\documentclass[journal]{IEEEtran}
\usepackage{amsmath,amssymb,graphicx,mathrsfs,epsfig}
\usepackage[algonl,ruled]{algorithm2e}
\usepackage{tikz}

\usepackage[latin1]{inputenc}
\usepackage{amsmath}
\usepackage{amsfonts}
\usepackage{float}

\newtheorem{thm}{Theorem}
\newtheorem{defn}{Definition}
\ifCLASSINFOpdf
\else
\fi
\begin{document}

\def \bw {\tilde{\mathbf{w}}}
\def \by {\tilde{\mathbf{y}}}
\def \bx {\tilde{\mathbf{x}}}
\def \bz {\tilde{\mathbf{Z}}}
\def \bX {\tilde{\mathbf{X}}}
\def \xx {\mathbf{x}}
\def \I {\mathbf{I}}
\def \K {\mathbf{k}}
\def \yy {\mathbf{y}}
\def \zz {\mathbf{z}}
\def \ee {\mathbf{e}}
\def \gg {\mathbf{g}}
\def \ww {\mathbf{w}}
\newcommand{\N}{{\mathbb N}}
\newcommand{\E}{{\mathbb E}}
\def \amax {\operatorname{argmax}}
\def \amin {\operatorname{argmin}}
\def \muu  {\mbox{\boldmath $\mu$}}
\def \xii  {\mbox{\boldmath $\xi$}}
\def \alphaa  {\mbox{\boldmath $\alpha$}}

%
\title{$K$-Plane Regression}
%
%
%

\author{Naresh~Manwani,
        P.~S.~Sastry,~\IEEEmembership{Senior~Member,~IEEE}
\thanks{Naresh Manwani and P. S. Sastry are with the Department
of Electrical Engineering, Indian Institute of Science, Bangalore 560012,
India e-mail: (naresh,sastry@ee.iisc.ernet.in).}}

\maketitle

\begin{abstract}
In this paper, we present a novel algorithm for piecewise linear regression which can learn continuous as well as
discontinuous piecewise linear functions. The main idea is to repeatedly partition the data and learn a liner model in each partition. While a simple algorithm incorporating this idea does not work well, an interesting modification
results in a good algorithm.
The proposed algorithm is similar in spirit to $k$-means clustering algorithm.
We show that our algorithm can also be viewed as an EM algorithm for maximum likelihood estimation of parameters under a
reasonable probability model. We empirically demonstrate the effectiveness of our approach by
comparing its performance with the state of art regression learning algorithms on some real
world datasets.
\end{abstract}


%
\IEEEpeerreviewmaketitle

\section{Introduction}
In a regression problem, given the training dataset containing pairs of multi-dimensional feature vectors and
corresponding real-valued target outputs, the task is to learn a function that captures the relationship between
feature vectors and their corresponding target outputs.

Least square regression  and support vector regression  are well known and
generic approaches for regression learning problems  \cite{Bishop,Hastie,smola1998,Zhu2008,Jin2005}.
In the least squares approach, nonlinear regression functions can be learnt by using
user-specified fixed nonlinear mapping of feature vectors from original space to some suitable high dimensional space though
this could be computationally expensive.
In support vector regression (SVR),  kernel functions are used for nonlinear problems.
Using a nonlinear kernel function, SVR implicitly transforms the examples to some high dimensional space
and finds a linear regression function in the high dimensional space. SVR has a large margin flavor
 and has well studied performance guarantees.
In general, SVR solution is not easily interpretable in the original feature space for nonlinear problems.

A different approach to learning a nonlinear regression function is to approximate the target function
by a piecewise linear function. Piecewise linear approach for regression problems provides better understanding
of the behavior of the regression surface in the original feature space as compared to the kernel-based approach of SVR.
In piecewise linear approaches,
the feature space is partitioned into disjoint regions and for every partition a linear regression function is learnt.
The goal here is to simultaneously estimate the optimal partitions and linear model for each partition.
This problem is hard and is computationally intractable \cite{Paoletti2007}.

The simplest piecewise linear function is either a convex or a concave
piecewise linear function which is represented as a maximum
or minimum of affine functions.
A generic piecewise linear regression function can be represented as a
sum of these convex/concave piecewise linear functions \cite{Breiman,Shuning,Magnani2009}.

In this paper we present a novel method of learning piecewise linear regression functions. In contrast to all the existing methods,
our approach is capable of learning discontinuous functions also. We show, through empirical studies, that this algorithm
is attractive in comparison to the SVR approach as well as the hinge hyperplanes method which is among the best algorithms
 for learning piecewise linear functions.

Existing approaches for piecewise linear regression learning can be broadly classified into two classes.
In the first set of approaches one assumes a specific form for the function and
estimates the parameters. Form of a regression function can be fixed by fixing the number of hyperplanes and
fixing the way these hyperplanes are combined to approximate the regression surface.
In the second set of approaches, the form of the regression function is not fixed apriori.

In fixed structure approaches we search over a parametrized family of piecewise linear regression functions and the
parameters are learnt by solving an optimization problem
to, typically,  minimize the sum of the squared errors. Some examples of such methods are
 mixture of experts and hierarchical
mixture of experts \cite{Nowlan1,Waterhouse,Jordan94} models.

In the  set of approaches where no fixed structure is assumed, {\em regression tree} \cite{Breiman1984,Jayadeva} is the most
widely used method. A regression tree is built by binary or multivariate recursive partitioning in a greedy fashion.
Regression trees split the feature space at every node in such a way that fitting a linear regression function to
each child node will minimize the sum of squared errors.
This splitting or partitioning is then applied to each of the child nodes.
The process continues until the number of data points at a node reaches a user-specified
minimum size or the error becomes smaller than some tolerance limit.
 In contrast to decision trees
where leaf nodes are assigned class labels, leaf nodes in regression trees are associated with
linear regression models. Most of the algorithms for learning regression trees are greedy in nature.
At any node of the tree, once a hyperplane is learnt to split the feature space, it can not be altered by
 any of its child nodes. The greedy nature of the method can result in convergence to a suboptimal solution.

A more refined regression tree approach is {\em hinging hyperplane} method \cite{Breiman,Pucar} which overcomes
several drawbacks of regression tree approach. A hinge function is defined as maximum
or minimum of two affine functions \cite{Breiman}. In the hinging hyperplane approach, the regression function
is approximated as a sum of these hinge functions where the number of hinge functions are not fixed {\em apriori}.
The algorithm starts with fitting a hinge function on the training data using the hinge finding algorithm \cite{Breiman}.
Then, residual error is calculated for every example and based on this a new hinge function may be added to the model (unless
we reach the maximum allowed number of hinges). Every time a new hinge function is added, its parameters are found by
fitting the residual error. This algorithm overcomes the greedy nature of regression tree approach
by providing a mechanism for re-estimation of parameters of each of the earlier hinge function whenever a new hinge is added.
 Overall, hinge hyperplanes algorithm tries to learn an optimal regression tree, given the training data.

A different greedy approach for piecewise linear regression learning
is {\em bounded error approach} \cite{Amaldi,Alberto2003,Alberto2005}.
In bounded error approaches, for a given bound ($\epsilon >0$) on the tolerable error,
the goal is to learn a piecewise linear regression function such that for every point in the training set,
the absolute difference between the target value and the predicted value
is less than $\epsilon$. This property is called bounded error property. Greedy heuristic algorithms \cite{Alberto2003,Alberto2005}
have been proposed to find such a piecewise linear function. These algorithms start with finding a linear regression function which
should satisfy the bounded error property for as many points in the training set as possible. This problem is known as
{\em maximum feasible sub-system problem} (MAX-FS) and is shown to be NP-hard \cite{Amaldi}. MAX-FS problem is repeated on the
remaining points until all points are exhausted. So far, there are no theoretical results to support the quality of the solution
of these heuristic approaches.

Most of the existing approaches for learning regression functions find a continuous approximation for the
regression surface even if the actual surface is discontinuous. In this paper, we present a  piecewise
linear regression algorithm which is able to learn both continuous as well as discontinuous functions.

We start with a simple algorithm that is similar, in spirit, to the $k$-means clustering algorithm.
The idea is to repeatedly keep partitioning the training data and learning a hyperplane for each partition. In each such iteration,
after learning the hyperplanes, we repartition the feature vectors so that all feature vectors in a partition have least prediction
error with the hyperplane of that partition.
We call it $K$-plane regression algorithm. Though we are not aware of any literature where such a method is
explicitly proposed and investigated for learning regression functions, similar ideas have been proposed in related contexts.
For example, a similar problem is addressed in the system identification literature \cite{Amaldi}.
A probabilistic version of such an idea was discussed under the title
{\em mixtures of multiple linear regression} \cite[Chapter 14]{Bishop}.

This $K$-plane regression algorithm is
attractive because it is conceptually very simple. However, it suffers from some serious drawbacks in terms of
convergence to non-optimal solutions, sensitivity to additive noise and lack of model function.
We discuss these issues and based on this insight propose new and modified $K$-plane regression
algorithm. In the modified algorithm also we keep repeatedly partition the data and learning a linear model for
each partition. However, we try to separately and simultaneously learn the centers of the partitions and the corresponding linear
models. Through empirical studies we show that this algorithm is very effective for learning piecewise linear regression
surfaces and it compares favorably with other state-of-art regression function learning methods.

The rest of the paper is organized as follows. In Section \ref{sec:k-plane} we discuss
$K$-plane regression algorithm, its drawbacks and possible reasons behind them. We then propose
modified $K$-plane regression algorithm in Section \ref{sec:modi-k-plane}.
We also show that modified $K$-plane regression algorithm monotonically decreases the error function after every iteration.
In Section \ref{sec:EM} we show the equivalence of our algorithm with
an EM algorithm in limiting case. Experimental results are given in Section~\ref{sec:experiments}.
We conclude the paper in Section \ref{sec:conclusion}.

\section{$K$-Plane Regression}
\label{sec:k-plane}
We begin by defining a $K$-piecewise affine function. We use the notation that a hyperplane in $\Re^d$ is
parametrized by $\bw=[\ww^T~~b]^T\in \Re^{d+1}$ where $\ww \in \Re^d$ and $b\in \Re$.
\begin{defn}\label{defn:1}
A function $f:\Re^d \rightarrow \Re$, is called \textbf{$K$-piecewise affine} if there exists a set of
$K$ hyperplanes with parameters $(\ww_1,b_1),\ldots, (\ww_K,b_K)\in \Re^{d+1}$ $((\ww_i,b_i)\neq (\ww_j,b_j), i\neq j)$,
 and sets $\tilde{S}_1,\ldots, \tilde{S}_K\subset \Re^d$
(which form a partition of $\Re^d$), such that, $f(\xx)= \ww_k^T\xx+b_k,\forall \xx \in \tilde{S}_k$.
\end{defn}
From the definition above, it is clear that $(\ww_j^T\xx+b_j-f(\xx))^2 \geq (\ww_k^T\xx+b_k-f(\xx))^2 = 0, \forall \xx \in \tilde{S}_k,\forall j\neq k$.
Also, note that a $K$-piecewise affine function may be discontinuous.

\subsection*{$K$-Plane Regression}
Let $S=\{(\xx_1 , y_1),\ldots,(\xx_N,y_N)\}$ be the training dataset, where
$(\xx_n,y_n) \in \Re^d\times\Re$. Let $\bx_n=[\xx_n^T\;\;\;1]^T,\;n=1\ldots N$.
$K$-plane regression approach tries to find a pre-fixed number of hyperplanes such that
each point in the training set is close to one of the hyperplanes.
Let $K$ be the number of hyperplanes.
Let $\bw_k,\;k=1\ldots K$, be the parameters of the hyperplanes.
$K$-plane regression minimizes the following error function.
\begin{equation}
 \nonumber E(\Theta)  =  \sum_{n=1}^N \; \min_{k \in \{ 1 \ldots K \} } (\bw_k^T\bx_n - y_n)^2
\end{equation}
where $\Theta=\{\bw_1,\ldots,\bw_K\}$. Given the parameters of $K$ hyperplanes, $\bw_1 \ldots \bw_K$,
define sets $S_k,\;k=1\ldots K$, as $S_k:=\{\xx_n\;|\;k=\amin_{j\in\{1,\ldots K\}}(\bw_j^T\bx_n-y_n)^2\}$ where we break
ties by putting $\xx_n$ in the set $S_k$ with least $k$. The sets $S_k$ are disjoint. We can now write $E(\Theta)$ as
\begin{equation}\label{cost2}
E(\Theta)= \sum_{k=1}^K \; \sum_{\xx_n \in S_k}(\bw_k^T\bx_n-y_n)^2
\end{equation}
If we fix all $S_k$, then $\bw_k$ can be found by minimizing (over $\bw$) $\sum_{\xx_n \in S_k}(\bw^T\bx_n-y_n)^2$.
However, in $E(\Theta)$ defined in equation (\ref{cost2}), the sets $S_k$ themselves are function
of the parameter set $\Theta=\{\bw_1,\ldots,\bw_K\}$.

To find $\Theta$ which minimize $E(\Theta)$ in (\ref{cost2}), we can have an EM-like algorithm
as follows. Let, after $c^{th}$ iteration, the parameter set be $\Theta^c$. Keeping $\Theta^c$ fixed,
we first find sets $S_k^c=\{\xx_n\;|\;k=\amin_{j\in\{1,\ldots K\}}(\bx_n^T\bw_j^c-y_n)^2\},\;k=1\ldots K$.
Now we keep these sets $S_k^c,\;k=1\ldots K$, fixed.
Thus the error function becomes
\begin{eqnarray}
\nonumber E^c(\Theta)&=&\sum_{k=1}^K \; \sum_{\xx_n \in S_k^c} (\bw_k^T\bx_n-y_n)^2=\sum_{k=1}^K E_k^c(\bw_k)
\end{eqnarray}
where superscript $c$ denotes the iteration and hence emphasizes the fact that the error function is evaluated
by fixing the sets $S_k^c,k=1\ldots K$, and
\begin{equation}\label{eq:subprob}
 E_k^c(\bw_k)=\sum_{\xx_n \in S_k^c} (\bw_k^T\bx_n-y_n)^2.
\end{equation}
Thus,
minimizing $E^c(\Theta)$ with respect to $\Theta$ boils down to minimizing each of $E_k^c(\bw_k)$ with respect to $\bw_k$.
For every $k\in\{1,\ldots, K\}$, a new weight vector $\bw_k^{c+1}$ is found using standard
linear least square solution as follows.
\begin{eqnarray}\label{eq:subprob1}
\nonumber  \bw_k^{c+1} &=& \amin_{\bw_k} \sum_{\xx_n \in S_k^c} (\bw_k^T\bx_n-y_n)^2\\
&=& \Big{[}\sum_{\xx_n \in S_k^c}\bx_n \bx_n^T\Big{]}^{-1}\big{[}\sum_{\xx_n \in S_k^c}y_n\bx_n\big{]}
\end{eqnarray}
Now we fix $\Theta^{c+1}$ and find new sets $S_k^{c+1},k=1\ldots K$, and so on.
We can now summarize $K$-plane regression algorithm. We first find sets $S_k^c,\;k=1\ldots K$,
for iteration $c$ (using $\bw_k^c,\;k=1\ldots K$). Then for each $k=1 \ldots K$, we find $\bw_k^{c+1}$
(as in equation (\ref{eq:subprob1})) by minimizing $E_k^c(\bw_k)$
which is defined in equation (\ref{eq:subprob}).
We keep on repeating these two steps until there is no significant decrement
in the error function $E(\Theta)$. $E(\Theta)$ does not change when the weight vectors
do not change or sets $S_k,\;k=1\ldots K$, do not change. The complete $K$-plane regression approach is described more formally in Algorithm \ref{algo1}.
\begin{algorithm}[t]
\caption{$K$-plane regression}
\label{algo1}
\KwIn{$\{(\xx_1,y_1),\ldots,(\xx_N,y_N)\}$}
\KwOut{$\{\bw_1\ldots \bw_K\}$}
\Begin
{
\textit{Step 1:} Initialize $\bw_k^0,k=1\ldots K$, Initialize $c=0$.\\
\textit{Step 2:} Find sets $S_k^c,k=1\ldots K$
\begin{itemize}
 \item[] $S_k^c= \{\xx_n\;|\;k=\amin_{j\in\{1\ldots K\}}(\bx_n^T\bw_j^c-y_n)^2\}$
\end{itemize}
\textit{Step 3:} Find $\ww_k^{c+1},k=1\ldots K$, as follows
\begin{itemize}
 \item[] $\bw_k^{c+1} = \Big{[}\sum_{\xx_n \in S_k^c}\bx_n \bx_n^T\Big{]}^{-1}\big{[}\sum_{\xx_n \in S_j^c}y_n\bx_n\big{]}$
\end{itemize}
\textit{Step 4:} Find sets $S_k^{c+1},k=1\ldots K$
\begin{itemize}
 \item[] $S_k^{c+1}= \{\xx_n\;|\;k=\amin_{j\in\{1\ldots K\}}(\bx_n^T\bw_j^{c+1}-y_n)^2\}$
\end{itemize}
\textit{Step 5:} Termination Criteria\\
      \uIf{$S_k^{c+1} = S_k^c,k=1\ldots K$}{
	    stop\\
      }
      \Else{
	    $c=c+1$\\
	    go to \textit{Step 3}\\
      }
}
\end{algorithm}

\subsection{Issues with $K$-Plane Regression}

In spite of its simplicity and easy updates, $K$-plane regression algorithm has some serious drawbacks
in terms of convergence and model issues.

\subsubsection*{1. Convergence to Non-optimal Solutions}

It is observed that the algorithm has serious problem of convergence to non-optimal solution.
Even when the data is generated from a piecewise linear function, the algorithm often fails to learn the structure of the target function.
\begin{figure}[t]
\begin{center}
\begin{tabular}{c}
   \includegraphics[scale=.37]{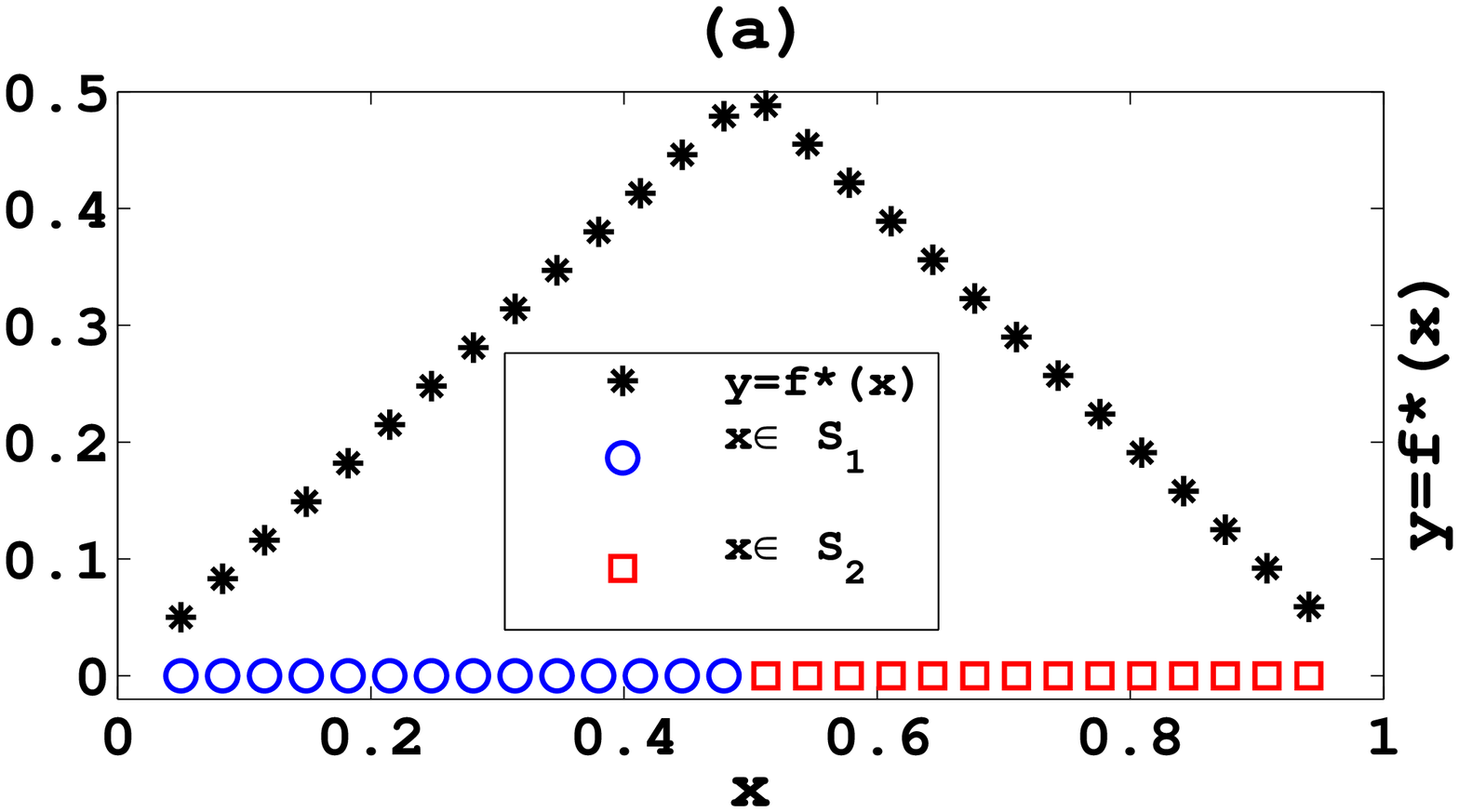}\\
 \includegraphics[scale=.37]{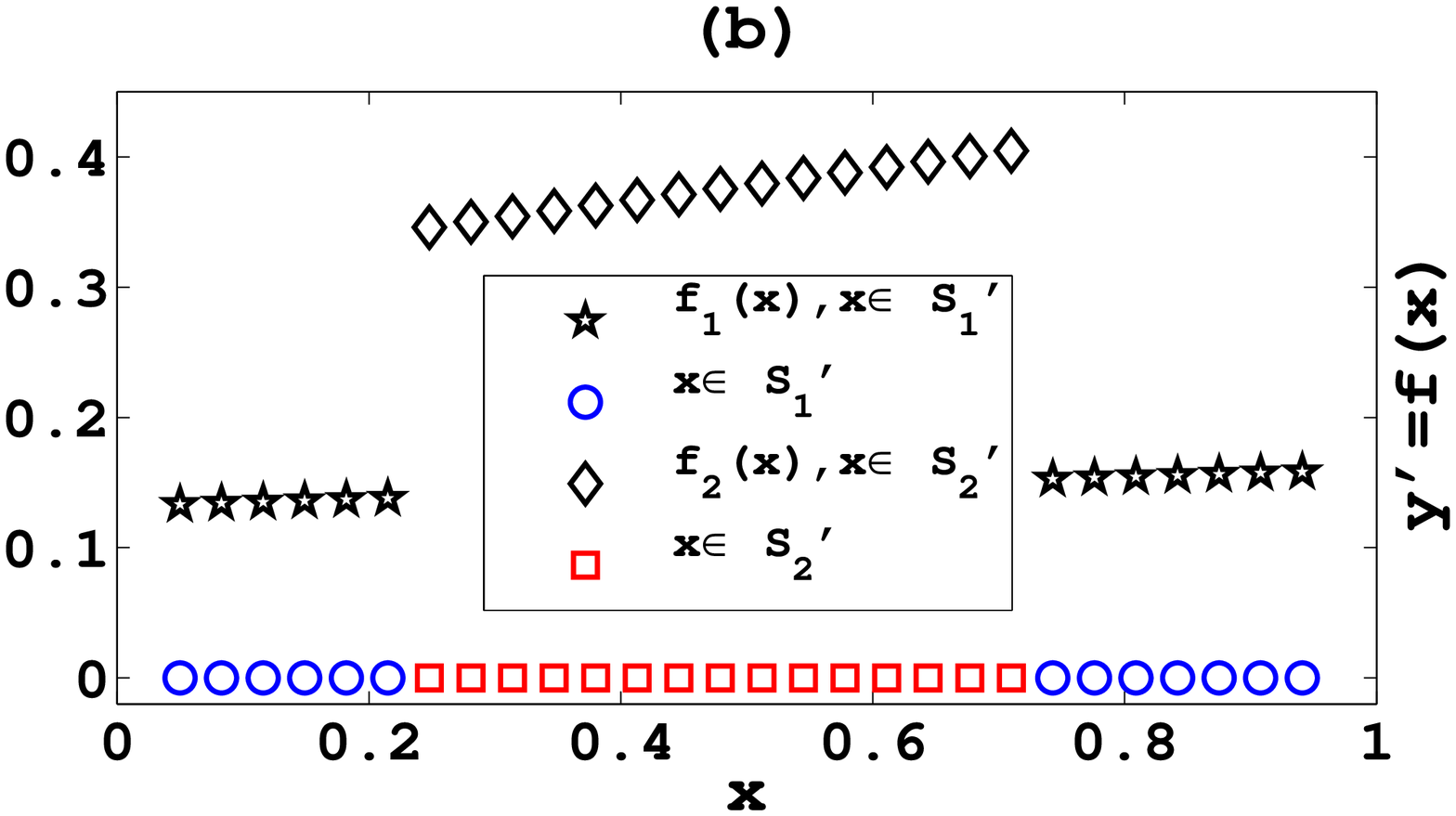}
 \end{tabular}
\caption{(a) Points sampled from a triangle shaped function $f^*(x)$, (b) function $f(x)$ learnt using $K$-plane regression algorithm given the points sampled from function $f^*(x)$.}
\label{nonconvergence1}
\end{center}
\end{figure}

Figure \ref{nonconvergence1}(a) shows points sampled from a concave (triangle shaped) 2-piecewise affine function on the real line.
At the horizontal axis, circles represent set $S_1$ and squares represent set $S_2$, where $S_1$ and $S_2$ constitute  the
correct partitioning of the training set in this problem. We see that convex hulls of sets $S_1$ and $S_2$ are disjoint.
This 2-piecewise affine function can be written (as per defn. \ref{defn:1}) by choosing $\tilde{S}_1,\;\tilde{S}_2$ to be
the convex hulls of $S_1$ and $S_2$.

Figure \ref{nonconvergence1}(b) shows the 2-piecewise linear function learnt using $K$-plane regression approach for a
particular initialization. $S_1'$ (represented as circles) and $S_2'$ (represented as squares) are sets corresponding
to the two lines in the learnt function. Here, $K$-plane regression algorithm completely misses the shape of the target
function. We also see that convex hulls of sets $S_1'$ and $S_2'$ intersect with each other.

\subsubsection*{2. Sensitivity to Noise}
It has been observed in practice that the simple $K$-plane regression algorithm is very sensitive to the additive noise
in the target values in training set. Under noisy examples, the algorithm performs badly. We illustrate it
later in Section~\ref{sec:experiments}.

\subsubsection*{3. Lack of Model Function}
The output of the $K$-plane regression algorithm is a set of $K$ hyperplanes.
But this algorithm does not provide a way to use these hyperplanes to predict the value for a given test point.
In other words, $K$-plane regression algorithm does not have any model function for prediction.
 We expand this issue in the next section.

\section{Modified $K$-Plane Regression}\label{sec:modi-k-plane}

As we have mentioned, given the training data, $\{(\xx_1,y_1),\ldots,(\xx_N,y_N)\}$,
the $K$-plane regression algorithm outputs $K$ hyperplanes,
$\bw_k^*,\;k=1\ldots K$. To convert this into a proper $K$-piecewise linear model in $\Re^d$, we also need to
have a $K$-partition of $\Re^d$ such that in the $k^{th}$ partition,
the appropriate model to use would be $\bw_k^*$.
We could attempt to get such a partition of $\Re^d$ by considering the convex hulls of
$S_k^*,\;k=1\ldots K$ (where $S_k^*=\{\xx_n\;|\;k=\amin_j(y_n-\bx_n^T\bw_j^*)^2\}$).
However, as we saw, under the $K$-plane regression, the convex hulls of such $S_k^*$ need not be disjoint.
Hence another method to get the required partition
is as follows. Let $\muu_k^*$ be the mean or centroid of $S_k^*$.
Then, for any point, $\xx \in \Re^d$, our prediction could be $\hat{y}=\xx^T\bw_j^*$,
where $j$ is such that $||\xx-\muu_j^*||\leq ||\xx-\muu_k^*||,\;\forall k\neq j$ (break ties arbitrarily).
This would define a proper model function with the hyperplanes obtained through $K$-plane regression. However, this
may not give good performance. Often, the convex hulls of sets, $S_k^*,\;k=1\ldots K$ (learnt using $K$-plane regression),
 have non-null intersection because each of these sets may
 contain points from different disjoint regions of $\Re^d$ (for example, see Figure \ref{nonconvergence1}).
In such cases, if we re-partition the training data using distances to different $\muu_j^*$, we may get sets
much different from $S_k^*$ and hence our final prediction on even training data may have large error.
The main reason for this problem with $K$-plane regression is that the algorithm is not really bothered about the geometry of
the sets $S_k$; it only focuses on $\bw_k$ to be a good fit for points in set $S_k$.
Moreover, in situations where same affine function works for two or more disjoint clusters,
$k$-plane regression will consider them as a single cluster as the objective function
of $k$-plane regression does not enforce that points in the same cluster should be
close to each other. As a result, the clusters learnt using $K$-plane regression approach will
have overlapping convex hulls and some times even their means may be very close to each other.
This may create problems during prediction. If we use the hyperplane whose corresponding cluster mean is closest to a point,
then we may not pick up the correct hyperplane.
This identification problem of $K$-plane regression approach results in poor performance.

Motivated by this, we modify $K$-plane regression as follows.
We want to simultaneously estimate $\bw_k,\;k=1\ldots K$ and $\muu_k,\;k=1\ldots K$,
such that if, $\bw_k$ is a good fit for $k^{th}$ partition, all the
points in $k^{th}$ partition should be closer to $\muu_k$ than any other $\muu_j^*$.
Intuitively, we can think of $\muu_k$ as center of the (cluster or) set of points $S_k$.
However, as we saw from our earlier example, if we simply make
$\muu_k$ as the centroid of the final $S_k$ learnt, all the earlier problem still remain.
Hence, in the modified $K$-plane regression, we try to
independently learn both $\bw_k$ and $\muu_k$ from the data.
To do that, we add an extra term to the objective function of $K$-plane regression approach
which tries to ensure that all the points of same cluster are close together.


As earlier, let the number of hyperplanes be $K$.
Here, in the modified $K$-plane regression, we have to learn $2K$  parameter vectors.
Corresponding to $k^{th}$ partition, we have two parameter vectors, $\bw_k \in \Re^{d+1}$ and $\muu_k \in \Re^d$.
$\bw_k$ represents parameter vector of the hyperplane associated with the $k^{th}$ partition and
$\muu_k$ represents center of the $k^{th}$ partition. Note that we want to simultaneously learn
both $\bw_k$ and $\muu_k$ for every partition.

The error function minimized by modified $K$-plane regression algorithm is
\begin{equation}\label{error2}
 E(\Theta)= \sum_{n=1}^N  \; \min_{k\in \{1,\ldots,K\}} \big{[}(\bw_k^T\bx_n-y_n)^2 + \gamma ||\xx_n-\muu_k||^2\big{]}
\end{equation}
where $\Theta=\{(\bw_1,\muu_1),\ldots ,(\bw_K,\muu_K)\}$ and $\gamma$ is a user defined parameter
which decides relative weight of the two terms.

Given $\Theta$, we define sets $S_k,\;k=1\ldots K$, as
\begin{eqnarray}
 \nonumber S_k:=\Big{\{}\xx_n\;|\;k=\amin_{j}\big{[}(\bw_j^T\bx_n-y_n)^2+\gamma ||\xx_n-\muu_j||^2\big{]}\Big{\}}
\end{eqnarray}
where we break
ties by putting $\xx_n$ in the set $S_k$ with least $k$. The sets $S_k$ are disjoint. We can now write $E(\Theta)$ as
\begin{equation}\label{error2-new}
E(\Theta)= \sum_{k=1}^K \sum_{\xx_n \in S_k}(\bw_k^T\bx_n-y_n)^2+\gamma||\xx_n-\muu_k||^2
\end{equation}

As can be seen from the above, now, for a data point, $\xx_n$ to be put in $S_k$, we not only need $\bw_k^T\bx_n$ to be close
to $y_n$ as earlier, but also need $\xx_n$ to be close to $\muu_k$, the `current center' of $S_k$. The motivation is that,
 under such a partitioning strategy, each $S_k$ would contain only points that are close to each other.
As we shall see later through simulations, this modification ensures that the algorithm performs well.
As an example of where this modification is important,  consider
 learning a piecewise linear model  which is given by same affine function
 in two (or more) disjoint regions in the feature space.
For any splitting of all examples from these two regions into two parts, there will be a good linear model that fits each of the
two parts. Hence, in the $K$-plane regression method, the $ E(\Theta)$ function (cf.eq.(\ref{cost2})) would be same for any
splitting of the examples from these two regions which means we would not learn a good model. However,
the modified $K$-plane regression approach will not treat all such splits as same because of the term involving $\muu_k$. This helps us
learn a proper piecewise linear regression function. We illustrate this in  Section~\ref{sec:experiments}.



Now consider finding $\Theta$ to minimize $E(\Theta)$ given by equation~(\ref{error2-new}). If we fix all
 $S_k$, then $\bw_k$ and $\muu_k$ can be found by minimizing (over $\bw,\;\muu$)
$\sum_{\xx_n \in S_k}(\bw^T\bx_n-y_n)^2+\gamma||\xx_n-\muu||^2$. However, in $E(\Theta)$ defined in equation~(\ref{error2-new}),
the sets $S_k$ themselves are functions of the parameter set $\Theta=\{(\bw_1,\muu_1),\ldots,(\bw_K,\muu_K)\}$.

To find $\Theta$ which minimize $E(\Theta)$ in (\ref{error2-new}), we can, once again,  have an EM-like
algorithm as follows. As earlier, let the parameter set after $c^{th}$ iteration be $\Theta^c$. Keeping $\Theta^c$ fixed,
we find the sets $S_k^c,\;k=1\ldots K$, as follows
\begin{equation}
  S_k^c=\big{\{}\xx_n\;|\;k=\amin_{j}[(\bx_n^T\bw_j^c-y_n)^2+\gamma ||\xx_n-\muu^c_j||^2]\big{\}}
\label{eq:s-k-c}
\end{equation}
Now we keep these sets $S_k^c$ fixed. Thus the error function becomes
\begin{eqnarray}
\nonumber E^c(\Theta)&=&\sum_{k=1}^K \sum_{\xx_n \in S_k^c} [(\bw_k^T\bx_n-y_n)^2 +\gamma ||\xx_n-\muu_k||^2]\\
\nonumber &=&\sum_{k=1}^KE^c(\bw_k,\muu_k)
\end{eqnarray}
where superscript $c$ denotes the iteration and emphasizes the fact that the error function is evaluated by fixing the sets $S_k^c,\;k=1\ldots K$.
Thus minimizing $E^c(\Theta)$ with respect to $\Theta$ boils down to minimizing each of $E_k^c(\bw_k,\muu_k)$
with respect to $(\bw_k,\muu_k)$. Each $E_k^c(\bw_k,\muu_k)$ is composed of two terms. The first term depends only on $\bw_k$
and it is the usual sum of squares of errors. The second term depends only on $\muu_k$ and it is the usual
cost function of $K$-means clustering. Thus,
the update equations for finding $\bw_k^{c+1}$ and $\muu_k^{c+1}$, $k=1\ldots K$, are
\begin{eqnarray}
\nonumber \bw_k^{c+1}&=& \amin_{\bw_k}\sum_{j=1}^K E^c(\bw_j,\muu_j)\\
\nonumber &=&\amin_{\bw} \sum_{\xx_n \in S_k^c} (\bw^T\bx_n-y_n)^2\\
& =&\Big{[}\sum_{\xx_n \in S_k^c}\bx_n \bx_n^T\Big{]}^{-1}\big{[}\sum_{\xx_n \in S_k^c}y_n\bx_n\big{]} \label{modified1}\\
\nonumber \muu_k^{c+1} & = & \amin_{\muu_k}\sum_{j=1}^K E^c(\bw_j,\muu_j)\\
 &=&\amin_{\muu} \sum_{\xx_n \in S_k^c}  ||\xx_n-\muu||^2 = \frac{1}{|S_k^c|}\sum_{\xx_n \in S_k^c} \xx_n \label{modified2}
\end{eqnarray}
\begin{algorithm}[t]
\caption{Modified $K$-plane regression}
\label{algo2}
\KwIn{$\{(\xx_1,y_1)\ldots (\xx_N,y_N)\}$}
\KwOut{$\{(\bw_1,\muu_1)\ldots (\bw_K,\muu_K)\}$}
\Begin
{
\textit{Step 1:} Initialize $(\ww_k^0,\muu_k^0),\;k=1\ldots K$. $c=0$.\\
\textit{Step 2:} Find $S_k^c,k=1\ldots K$, as follows
\begin{itemize}
 \item[] $S_k^c=\big{\{}\xx_n\;|\;k=\amin_{j}[(\bx_n^T\bw_j^c-y_n)^2 + \gamma ||\xx_n-\muu_j^c||^2]\big{\}}$
\end{itemize}
\textit{Step 3:} Find $\bw_k^{c+1},\muu_k^{c+1},k=1\ldots K$, as follows
\begin{eqnarray}
 \nonumber \bw_k^{c+1} &=& \Big{[}\sum_{\xx_n \in S_k^c}\bx_n \bx_n^T\Big{]}^{-1}\big{[}\sum_{\xx_n \in S_k^c}y_n\bx_n\big{]}\\
\nonumber \muu_k^{c+1} & = &\frac{1}{|S_k^c|}\sum_{\xx_n \in S_k^c} \xx_n
\end{eqnarray}
\textit{Step 4:} Find $S_k^{c+1},k=1\ldots K$, as follows
\begin{itemize}
 \item[] $S_k^{c+1}=\big{\{}\xx_n\;|\;k=\amin_{j}[(\bx_n^T\bw_j^{c+1}-y_n)^2 + \gamma ||\xx_n-\muu_j^{c+1}||^2]\big{\}}$
\end{itemize}
\textit{Step 5:} Termination Criteria\\
      \uIf{$S_k^{c+1}=S_k^c,\forall k$}{
	    stop\\
      }
      \Else{
	    $c=c+1$\\
	    go to \textit{Step 3}\\
      }
}
\end{algorithm}
Once we compute $\Theta^{c+1}$, we find new sets $S_k^{c+1},k=1\ldots K$, and so on.

In summary, the modified $K$-plane regression algorithm works as follows.
We first find sets $S_k^c,\;k=1\ldots K$, for iteration $c$ (using $(\bw_k^c,\muu_k^c),\;k=1\ldots K$) as given by eq.(\ref{eq:s-k-c}).
Then for each $k=1 \ldots K$, we find $(\bw_k^{c+1},\muu_k^{c+1})$ (as in equation (\ref{modified1}) and (\ref{modified2}))
by maximizing $E_k^c(\bw_k,\muu_k)$.
We keep on repeating these two steps until there is no significant decrement
in the error function $E(\Theta)$.
The complete description of modified $K$-plane regression approach is given in Algorithm \ref{algo2}.
\subsection*{Monotone Error Decrement Property}
Now we will show that modified $K$-plane regression algorithm monotonically decreases the error function defined by
equation (\ref{error2}).\footnote{ Note that this does not necessarily mean that we find the global minimum of the error function.
More importantly, we can not claim that minimizing the error as defined would lead to learning of a good piece-wise linear model.
 We note here that the simple $K$-plane regression algorithm also results in monotonic decrease in the error as defined for that algorithm
 even though it may not learn good models. However, the fact that the algorithm continuously decreases the error at each iteration,
is an important property for a learning algorithm.}

\begin{thm}\label{thm1}
Algorithm \ref{algo2} monotonically decreases the cost function given by equation (\ref{error2})
after every iteration.
\end{thm}
\begin{proof}
We have
\begin{equation}
 \nonumber E^c(\Theta^c)=\sum_{k=1}^K \sum_{\xx_n \in S_k^c}\big{[}(\bx_n^T\bw_k^c-y_n)^2 + \gamma ||\xx_n-\muu_k^c||^2\big{]}
\end{equation}
Given the sets $S_k^c,\;k=1\ldots K$, parameters $(\bw_k^{c+1},\muu_k^{c+1}),\;k=1\ldots K$, are found using equation (\ref{modified1}) and (\ref{modified2}), in the following way.
\begin{eqnarray}
 \nonumber \bw_k^{c+1} &=& \amin_{\bw_k} \sum_{\xx_n \in S_k^c} (\bw_k^T\bx_n-y_n)^2   \\
\nonumber \muu_k^{c+1} &=& \amin_{\muu_k}\sum_{\xx_n \in S_k^c} ||\xx_n-\muu_k||^2
\end{eqnarray}
Thus, we have
\begin{eqnarray}
\nonumber & &\sum_{\xx_n \in S_k^c} (\bx_n^T\bw_k^c-y_n)^2 \geq \sum_{\xx_n \in S_k^c} (\bx_n^T\bw_k^{c+1}-y_n)^2,\;k=1\ldots K \\
\nonumber & &\sum_{\xx_n \in S_k^c} ||\xx_n-\muu_k^c||^2 \geq  \sum_{\xx_n \in S_k^c} ||\xx_n-\muu_k^{c+1}||^2,\; k=1\ldots K
\end{eqnarray}
This will further give us
\begin{eqnarray}
\nonumber && \sum_{k=1}^K \sum_{\xx_n \in S_k^c}(\bx_n^T\bw_k^c-y_n)^2 + \gamma ||\xx_n-\muu_k^c||^2 \geq\\
\nonumber && \sum_{k=1}^K \sum_{\xx_n \in S_k^c} (\bx_n^T\bw_k^{c+1}-y_n)^2 + \gamma  ||\xx_n-\muu_k^{c+1}||^2\\
 &\Rightarrow&  E^c(\Theta^c) \geq  E^c(\Theta^{c+1}) \label{eq:thm1-eq1}
\end{eqnarray}
Given $\Theta^{c+1}$, sets $S_k^{c+1},k=1\ldots K$, are found as follows
\begin{eqnarray}
 \nonumber S_k^{c+1}=\big{\{}\xx_n\;|\;k=\amin_{j}[(\bx_n^T\bw_j^{c+1}-y_n)^2 + \gamma
  ||\xx_n-\muu_j^{c+1}||^2]\big{\}}
\end{eqnarray}
Using $S_k^{c+1},\;k=1\ldots K$, we can find $E^{c+1}(\Theta^{c+1})$, which is
\begin{eqnarray}
\nonumber &&E^{c+1}(\Theta^{c+1})\\
 \nonumber &=& \sum_{k=1}^K \sum_{\xx_n \in S_k^{c+1}}(\bx_n^T\bw_k^{c+1}-y_n)^2 + \gamma  ||\xx_n-\muu_k^{c+1}||^2\\
&=&\sum_{j,k=1}^K \sum_{\xx_n \in S_j^c\cap S_k^{c+1}}(\bx_n^T\bw_k^{c+1}-y_n)^2 + \gamma  ||\xx_n-\muu_k^{c+1}||^2\nonumber
\end{eqnarray}
By the definition of sets $S_k^{c+1}$, $\forall \xx_n \in S_k^{c+1}$, we have,
\begin{eqnarray}
\nonumber &&(\bx_n^T\bw_k^{c+1}-y_n)^2 + \gamma ||\xx_n-\muu_k^{c+1}||^2\\
\nonumber & \leq &(\bx_n^T\bw_j^{c+1}-y_n)^2 + \gamma
  ||\xx_n-\muu_j^{c+1}||^2,\;\forall j\neq k
\end{eqnarray}
which is also true for any $\xx_n \in S_k^{c+1}\cap S_j^c$. Thus
\begin{eqnarray}
&& E^{c+1}(\Theta^{c+1}) \nonumber \\
& \leq & \sum_{j,k=1}^K  \sum_{\xx_n \in S_j^c\cap S_k^{c+1} }(\bx_n^T\bw_j^{c+1}-y_n)^2 + \gamma  ||\xx_n-\muu_j^{c+1}||^2\nonumber \\
\nonumber &=&\sum_{j=1}^K \sum_{\xx_n \in S_j^c}(\bx_n^T\bw_j^{c+1}-y_n)^2 + \gamma  ||\xx_n-\muu_j^{c+1}||^2 \\
&= & E^c(\Theta^{c+1})\label{eq:thm1-eq2}
\end{eqnarray}
Combining (\ref{eq:thm1-eq1}) and (\ref{eq:thm1-eq2}), we get  $E^c(\Theta^c) \geq E^c(\Theta^{c+1})  \geq  E^{c+1}(\Theta^{c+1})$.
Which means after one complete iteration modified $K$-plane regression Algorithm decreases the error function.
\end{proof}

\section{EM View of Modified $K$-Plane Regression Algorithm}
\label{sec:EM}
Here, we show that modified $K$-plane regression algorithm presented in Section \ref{sec:modi-k-plane}
can be viewed as a limiting case of an EM algorithm.
In the general $K$-plane regression idea,
the difficulty is due to the following credit assignment problem.
When we decompose the problem into $K$ sub-problems, we do not know
which $\xx_n$ should be considered in which sub-problem.
We can view this as the missing information in the EM formulation.

Recall that $S=\{(\xx_1,y_1),\ldots,(\xx_N,y_N)\}$ is the training data set. In the EM framework,
$S=\{(\xx_1,y_1),\ldots,(\xx_N,y_N)\}$ can be thought of as incomplete data.
The missing data would be $\zz_n=[z_{n1}~\ldots~z_{nK}]^T$, where
$z_{nk}\in \{0,1\},\;\forall n,\;\forall k$, such that, $\sum_{k=1}^K z_{nk}=1,\;\forall n$.
Then $z_{nk}$ are defined as,
\begin{eqnarray}
 \nonumber z_{nk}=\begin{cases}
                   1,  & \text{if }k=\amin_{j} \big{[}(\bw_j^T\bx_n-y_n)^2+\gamma||\xx_n-\muu_j||^2\big{]}\\
                   0,  & \text{otherwise}
                  \end{cases}
\end{eqnarray}
This gives us the following probability model
\begin{eqnarray}\label{eq:model1}
\nonumber \begin{cases}
P(\xx_n,y_n|z_{nk}=1,\Theta)&= p(\xx_n,y_n|\bw_k,\muu_k)\\
&= p(\xx_n|\muu_k)p(y_n|\xx_n,\bw_k)\\
P(\xx_n,y_n|\zz_n,\Theta)&=\sum_{k=1}^K z_{nk}p(\xx_n|\muu_k)p(y_n|\xx_n,\bw_k)
\end{cases}
\end{eqnarray}
In our formulation, $\muu_k$ represents the center of the set of all $\xx$ for which the $k^{th}$ linear model is appropriate. Hence we
 take $p(\xx|\muu_k)=\mathcal{N}(\muu_k,\frac{\epsilon}{\gamma}I)$, a multivariate Gaussian in which the covariance matrix
is given by $\frac{\epsilon}{\gamma} I$, where $\frac{\epsilon}{\gamma}(\epsilon,\gamma>0)$ is a variance parameter,
and $I$ is the identity matrix. This covariance matrix is common for all $K$ components. We assume that
the target values given in the training set may be corrupted with zero mean Gaussian noise.  Thus,
for $k^{th}$ component, the target value is assigned using $\bw_k$ as $y=\bw_k^T\bx
+e$, where $e$ is Gaussian noise with mean 0 and
variance $\epsilon$. Variance $\epsilon$ is kept same for all $K$ components.
Thus, $p(y|\xx,\bw_k)=\mathcal{N}(\bw_k^T\bx_n,\epsilon)$, a Gaussian with mean $\bw_k^T\bx_n$ and
variance $\epsilon$. Thus,
\begin{eqnarray}
 \nonumber && p(\xx_n,y_n|\bw_k,\muu_k)= \mathcal{N}(\muu_k,{\frac{\epsilon}{\gamma}}I)
\mathcal{N}(\bw_k^T\bx_n,\epsilon)\\
\nonumber &=& \frac{\gamma^{\frac{d}{2}}}{(2\pi\epsilon)^{\frac{d}{2}}}
\exp\big{(}-\frac{\gamma}{2\epsilon}||\xx_n-\muu_k||^2\big{)}\times \\
\nonumber &&
\frac{1}{(2\pi\epsilon)^{\frac{1}{2}}}\exp\big{(}-\frac{1}{2\epsilon}(\bw_k^T\bx_n-y_n)^2\big{)}\\
\nonumber &=&\frac{1}{L} \exp\big{(}-\frac{1}{2\epsilon}[(y_n-\bw_k^T\bx_n)^2+\gamma||\xx_n-\muu_k||^2]\big{)}
\end{eqnarray}
where $L=\Big{[}\frac{(2\pi\epsilon)^{(d+1)}}{\gamma^d}\Big{]}^{\frac{1}{2}}$.
Note that $\epsilon$ and $\gamma$ are assumed to be fixed constant, instead of parameters to be
re-estimated. Thus, the density model for incomplete data becomes
\begin{eqnarray}\label{eq:model2}
\nonumber && p(\xx_n,y_n|\Theta)\\
&=&\sum_{k=1}^K \frac{\I_{\{k=a_n\}}}{L} \exp\Big{(}-\frac{(y_n-\bw_k^T\bx_n)^2+\gamma||\xx_n-\muu_k||^2}{2\epsilon}\Big{)} \nonumber \\
&=&\frac{1}{L} \exp\Big{(}-\frac{1}{2\epsilon}\min_k[(y_n-\bw_k^T\bx_n)^2+\gamma||\xx_n-\muu_k||^2]\Big{)}
\end{eqnarray}
Negative of the log-likelihood under the model given in equation (\ref{eq:model2})
is same as the error function minimized in the modified $K$-plane regression algorithm.
Hence, we can now  compute the EM iteration for maximizing log-likelihood computed from (\ref{eq:model1}).

However, the incomplete data log-likelihood under our probability model (\ref{eq:model2}) becomes non-differentiable due to the hard minimum function.
To get around this, we change the probability model for incomplete data into a mixture model with mixing coefficients as part of $\Theta$:
\begin{eqnarray}\label{eq:model3}
 p(\xx_n,y_n|\Theta)=\sum_{k=1}^K\frac{\alpha_k}{L} \exp\Big{(}-\frac{(y_n-\bw_k^T\bx_n)^2+\gamma ||\xx_n-\muu_k||^2}{2\epsilon}\Big{)}
\end{eqnarray}
where $\alpha_{k}=P(z_{nk}=1),\;\forall n$; $\alpha_k \geq 0,\;\sum_{k=1}^K \alpha_k =1$, and
$\Theta=\{(\alpha_1,\bw_1,\muu_1),\ldots,(\alpha_K,\bw_K,\muu_K)\}$. Note that here,
\begin{eqnarray}
\nonumber p(y_n|\xx_n,\Theta)& \propto & \sum_{k=1}^K \frac{\alpha_k \exp{(-\frac{\gamma}{2\epsilon}||\xx_n-\muu_k||^2})}{\sum_{j=1}^K\alpha_j \exp{(-\frac{\gamma}{2\epsilon}||\xx_n-\muu_j||^2)}}\times\\
\nonumber &&\exp{(-\frac{1}{2\epsilon}(y_n-\bw_k^T\bx_n)^2)}
\end{eqnarray}
which is same as the model
described in \cite{Xu1995} for a mixture of experts network.
The incomplete data log-likelihood given by (\ref{eq:model3}) will now be smooth and we can use EM algorithm to maximize the likelihood.
However, the model given in (\ref{eq:model3}) is somewhat different from the one in
equation (\ref{eq:model2}) which was used in Section \ref{sec:modi-k-plane}.

We, now derive the iterative scheme under EM framework using the model specified by
equation (\ref{eq:model1}) and (\ref{eq:model3}) and show that in the
limit $\epsilon \rightarrow 0$, the iterative scheme becomes the modified $K$-plane regression algorithm that we presented in Section
\ref{sec:modi-k-plane}.
\subsection{EM Algorithm}
We now describe EM algorithm with $S=\{(\xx_1,y_1),\ldots,(\xx_N,y_N)\}$ as incomplete data and $\bar{S}=\{(\xx_1,y_1,\zz_1),\ldots,(\xx_N,y_N,\zz_N)\}$ as complete data
and under the model specified by (\ref{eq:model1}) and (\ref{eq:model3}). The complete data log-likelihood is
\begin{eqnarray}
 \nonumber & & l_{\text{complete}}(\Theta;\bar{S})= \ln \big{[} \prod_{n=1}^N \prod_{k=1}^K
[P(\xx_n,y_n,z_{nk}|\Theta)]^{z_{nk}}\big{]}\\
\nonumber &=& \sum_{n=1}^N \sum_{k=1}^K z_{nk} \ln\big{[}P(z_{nk})P(\xx_n,y_n|z_{nk},\Theta)\big{]}\\
\nonumber &=& \sum_{n=1}^N \sum_{k=1}^K z_{nk} \big{[}\ln \alpha_k L - \frac{(y-\bw_k^T\bx_n)^2}{2\epsilon}-\frac{\gamma ||\xx_n-\muu_k||^2}{2\epsilon}\big{]}
\end{eqnarray}
\textbf{E-Step: }In E-Step, we find $Q(\Theta,\Theta^c)$ which is the expectation of complete
data log-likelihood.
\begin{eqnarray}
 \nonumber && \mathcal{Q}(\Theta,\Theta^c)=\E_{\{\zz_1,\ldots,\zz_N\}}\big{[}l_{\text{complete}}
(\Theta;\bar{S})|\Theta^c\big{]}\\
\nonumber &=& \sum_{n=1}^N \sum_{k=1}^K \big{[}\ln \alpha_k L - \frac{(y-\bw_k^T\bx_n)^2-\gamma ||\xx_n-\muu_k||^2}{2\epsilon}\big{]}\times\\
\nonumber && P(z_{nk}=1|\xx_n,y_n,\Theta^c)
\end{eqnarray}
\textbf{M-Step: }In the M-Step, we maximize $\mathcal{Q}(\Theta,\Theta^c)$ with respect to $\Theta$
to find out new parameter set $\Theta^{c+1}$. This will give us following update equations.
\begin{eqnarray}
 \nonumber \alpha_k^{c+1} &=& \frac{1}{N} \sum_{n=1}^N P(z_{nk}=1|\xx_n,y_n,\Theta^c)\\
\nonumber \bw_k^{c+1}&=& \big{[}\sum_{n=1}^N P(z_{nk}=1|\xx_n,y_n,\Theta^c)\bx_n\bx_n^T \big{]}^{-1}\times\\
\nonumber && \big{[}\sum_{n=1}^N P(z_{nk}=1|\xx_n,y_n,\Theta^c)y_n \bx_n \big{]}\\
\nonumber \muu_k^{c+1} &=& \frac{\sum_{n=1}^N P(z_{nk}=1|\xx_n,y_n,\Theta^c)\xx_n}{\sum_{n=1}^N P(z_{nk}=1|\xx_n,y_n,\Theta^c)}
\end{eqnarray}
where $P(z_{nk}=1|\xx_n,y_n,\Theta^c)$ is given by
\begin{eqnarray}
 \nonumber && P(z_{nk}=1|\xx_n,y_n,\Theta^c)\\
 \nonumber &=&  \frac{\alpha_k^c p(\xx_n,y_n|z_{nk}=1,\Theta^c)}{\sum_{j=1}^K \alpha_j^c p(\xx_n,y_n|z_{nk}=1,\Theta^c)}\\
\nonumber &=&  \frac{\alpha_k^c \exp\big{(}-\frac{1}{2\epsilon}[(y_n-\bx_n^T\bw_k^c)^2+\gamma ||\xx_n-\muu_k^c||^2]\big{)}}{\sum_{j=1}^K \alpha_j^c \exp\big{(}-\frac{1}{2\epsilon}[(y_n-\bx_n^T\bw_j^c)^2+\gamma ||\xx_n-\muu_j^c||^2]\big{)}}
\end{eqnarray}
\subsection{Limiting Case ($\epsilon \rightarrow 0$)}
Now consider $\lim_{\epsilon \rightarrow 0} P(z_{nk}=1|\xx_n,y_n,\Theta^c)$. Let $a_n^c=\amin_{j\in\{1,\ldots K\}} \big{[}(y_n-\bx_n^T\bw_j^c)^2 +\gamma ||\xx_n-\muu_k^c||^2\big{]}$. When $\epsilon \rightarrow 0$, then in the denominator, the term corresponding to index $a_n^c$ will go to zero most slowly and hence $\lim_{\epsilon \rightarrow 0} P(z_{nk}=1|\xx_n,y_n,\Theta^c)=\I_{\{k=a_n^c\}}$,
where $\I_{\{k=a_n^c\}}=1$ if $k=a_n^c$ and zero otherwise. In this limiting case, the EM updates of $\bw_k$ and $\muu_k$ will be same as updates of modified $K$-plane regression algorithm.

\section{Experiments}\label{sec:experiments}
In this section we present empirical results to show the effectiveness of modified $K$-plane regression approach.
We demonstrate how the learnt functions differ among various regression approaches using two synthetic problems.
We test the performance of our algorithm on several real datasets also.
We compare our approach with hinging hyperplane algorithm which is the best state-of-art regression tree algorithm and
 with support vector regression (SVR) which is among the best generic regression approaches
today.

\subsection*{Dataset Description}
The two synthetic datasets are generated as follows:
\begin{enumerate}
 \item \textbf{Problem 1: }In this, points are uniformly sampled from the interval $[0~~~5]$.
Then, for every point $x$ the target values $y$ are assigned as $y=f(x)+\epsilon$, where
\begin{eqnarray}
 \nonumber f(x)=\begin{cases}
                 x,   & \text{if } 0\leq x<1\\
2-x, & \text{if } 1\leq x <2 \\
\frac{1}{3}(7-2x),  & \text{if } 2\leq x < 3.5\\
\frac{1}{3}(2x-7),  & \text{if } 3.5\leq x \leq 5
                \end{cases}
\end{eqnarray}
and $\epsilon$ is a Gaussian random variable with zero mean and variance 0.01.
500 points are generated for training and 500 points are generated for testing.
\item \textbf{Problem 2: }Points are uniformly sampled from the interval $[0~~~3]$.
Then, for every point $x$ the target values $y$ are assigned as $y=f(x)$, where
\begin{eqnarray}
 \nonumber f(x)=\begin{cases}
                 x,   & \text{if } 0\leq x<1\\
1, & \text{if } 1\leq x <2 \\
x, & \text{if } 2\leq x \leq 3
                \end{cases}
\end{eqnarray}
We also generate $y'$ as $y'=f(x)+\epsilon$, where $\epsilon$ is a Gaussian random variable with zero mean
and variance 0.01. 300 points are generated for training and 300 points are generated for testing.
\end{enumerate}
Note that both the above functions are discontinuous.

We also present the experimental comparisons on 4 `real' datasets downloaded
from UCI ML repository \cite{Asuncion+Newman:2007} which are described in Table \ref{table:details}.
In our simulations, we scale all feature values to the range of $[-1~~ 1]$.

\begin{table}[h]
\begin{center}
 \begin{tabular}{|p{1.2in}|p{.7in}|p{.7in}|}
\hline
Data set                & Dimension &   \# Points  \\ \hline
Boston Housing          & 13        &   506        \\ \hline
Abalone                 & 8         &   4177       \\ \hline
Auto-mpg                &  7        &   392        \\ \hline
Computer activity       &  12      &   8192       \\ \hline
\end{tabular}
\caption{Details of real world datasets used from UCI ML repository.} \label{table:details}
\end{center}
\end{table}

\subsection*{Experimental Setup}We implemented $K$-plane regression and modified $K$-plane regression algorithms
in MATLAB.\footnote{
For $K$-plane regression, there is no specified model function which
can be used to predict the value for a test point. In our simulations, to assign value for any test point using $K$-plane regression,
we use the same methodology as modified $K$-plane regression approach. That is, using the $\bw_k$ learnt, we obtain sets $S_k$ as explained in
Section~\ref{sec:k-plane}; then we find the $k$ such that centroid of $S_k$ is closest to the test point and
then use that $\bw_k$ to predict the target.}
We have also implemented hinging hyperplane method in MATLAB.
For support vector regression, we used Libsvm \cite{Chih2001} code.
All the simulations are done on a PC (Core2duo, 2.3GHz, 2GB RAM).

Modified $K$-plane regression has one user defined parameter which is $\gamma$.
We search for the best value of $\gamma$ using 10-fold cross validation
and use that value in our simulations.
Both $K$-plane regression and modified $K$-plane regression approaches require $K$ (number of hyperplanes)
to be fixed apriori. In our experiments, we change the value of $K$ from 2 to 5.
Similarly, in hinging hyperplane method, maximum number of hinge functions should be specified.
In our simulations, this number is varied from 1 to 5.
 Support vector regression has three
user defined parameters: penalty parameter $C$, width parameter $\sigma$ for Gaussian kernel and tolerance parameter $\epsilon$.
Best values for these parameters are found using 10-fold cross-validation and the results reported are with these parameters.

\subsection*{Simulation Results: Synthetic Problems}

\textbf{Problem 1: }Figure \ref{fig:1d-example1} shows functions learnt using different approaches on problem 1 and
Table \ref{Table:problem1} shows MSE achieved with different approaches on a test set. Hinge hyperplane approach
and support vector regression (SVR) methods give continuous approximations to the function $f$ (see Figs. \ref{fig:1d-example1}(c) and \ref{fig:1d-example1}(d)).
While SVR gets the  shape of the function well, the function learnt using SVR is not piecewise linear.
Figure \ref{fig:1d-example1}(e) shows 4-piecewise affine function learnt using $K$-plane regression method.
We see that the $K$-plane regression
approach completely misses the shape of the function which results in a very high MSE. In contrast,
 as can be seen from figure \ref{fig:1d-example1}(f), modified $K$-plane regression approach learns
the discontinuous function $f$ exactly (even though the function values given in training set are noisy).

Recall that in modified K-plane regression, we essentially partition the data and learn a hyperplane as well as a `center' or `mean' (which was called
$\muu_k$ in the algorithm) for each partition. The target function in this example has four linear pieces. If we got the exact partitioning of the
input space then the ideal centers would be (0.5,1.5,2.75,4.25).
The means learnt using modified $K$-plane regression approach are (0.495,1.495,2.745,4.25).
 This example demonstrates that our modified $K$-plane regression algorithm is
robust to additive noise, and that it can learn discontinuous functions also well. This example also shows that the
simple-minded $K$-plane regression performs poorly when there is noise in the training set.
\begin{table}
\begin{center}
  \begin{tabular}{|p{.9in}|p{1.3in}|p{.45in}|}
\hline
Method      &     Parameters        &      MSE          \\ \hline
$K$-plane   &     \# hyperplanes = 4 &    0.0557            \\ \hline
Modified $K$-plane &     \# hyperplanes = 4 &   0.037       \\ \hline
Hinge Hyperplane & \# hinges = 6 &   0.0451\\ \hline
SVR  &$C=64$, $\sigma=16$, $\epsilon=2^{-5}$ &   0.013 \\
\hline
\end{tabular}
\caption{MSE of different regression approaches on problem 1.} \label{Table:problem1}
\end{center}
\end{table}

\begin{figure*}
 \begin{center}
  \begin{tabular}{ccc}
 \includegraphics[scale=.38]{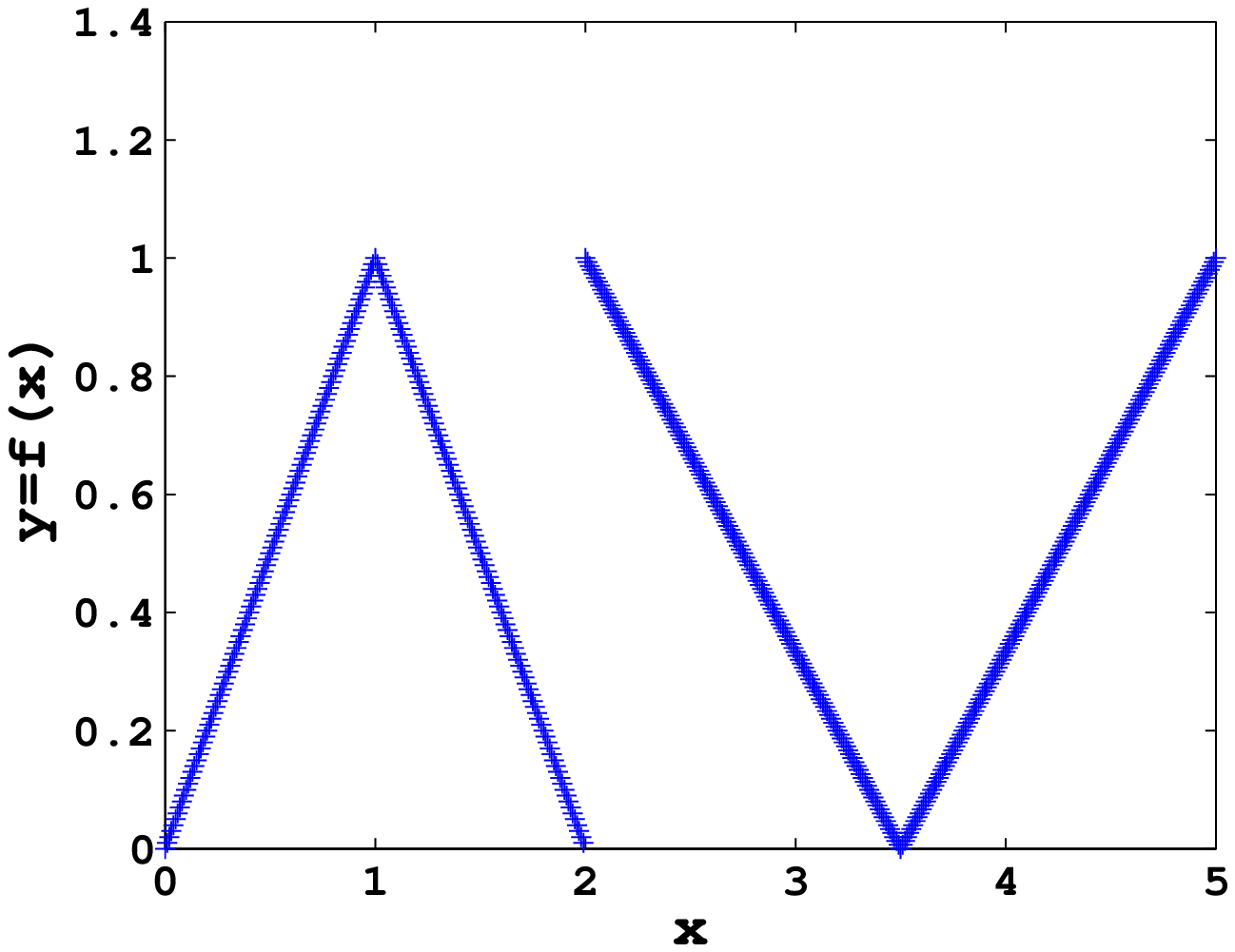}  &  \includegraphics[scale=.38]{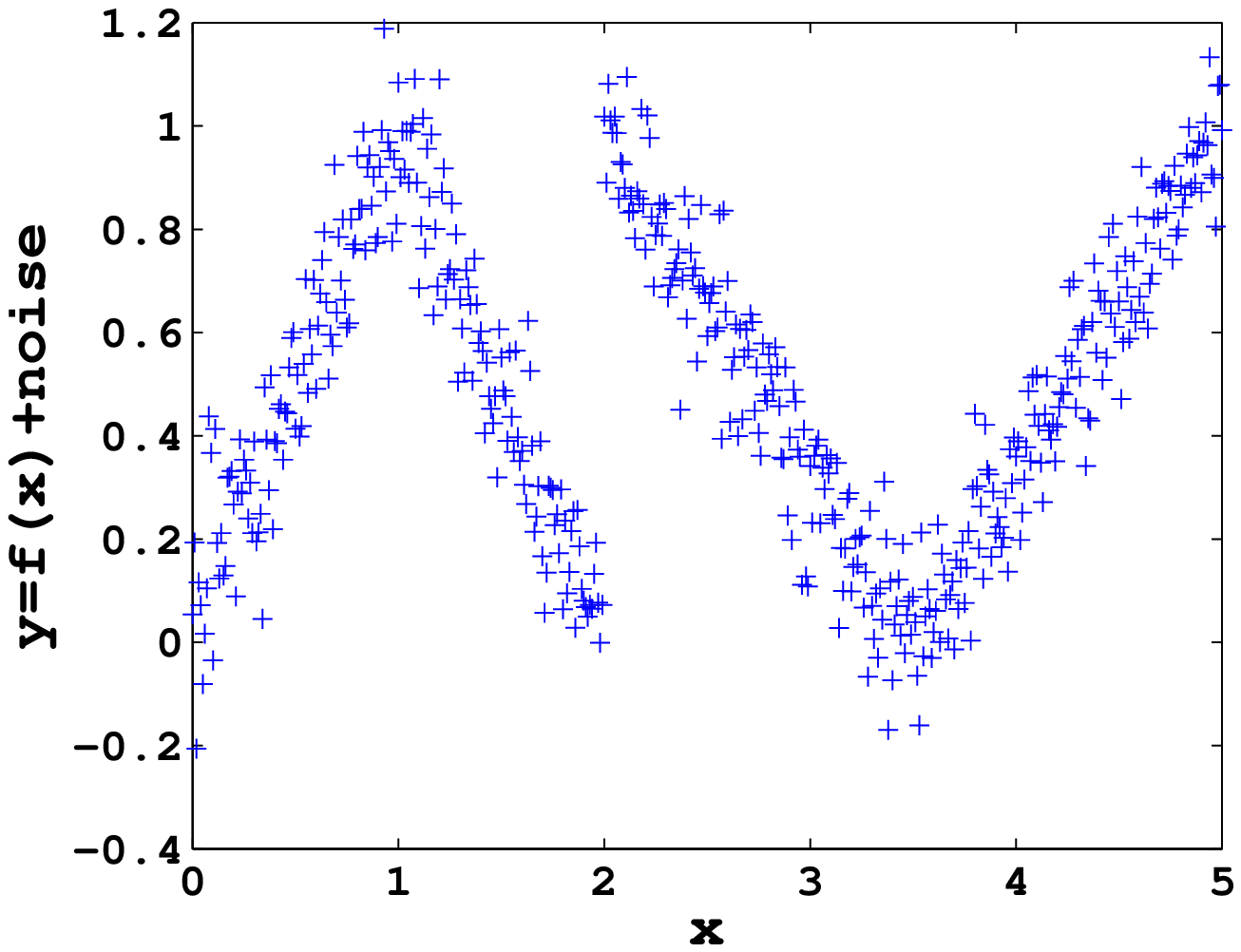} & \includegraphics[scale=.38]{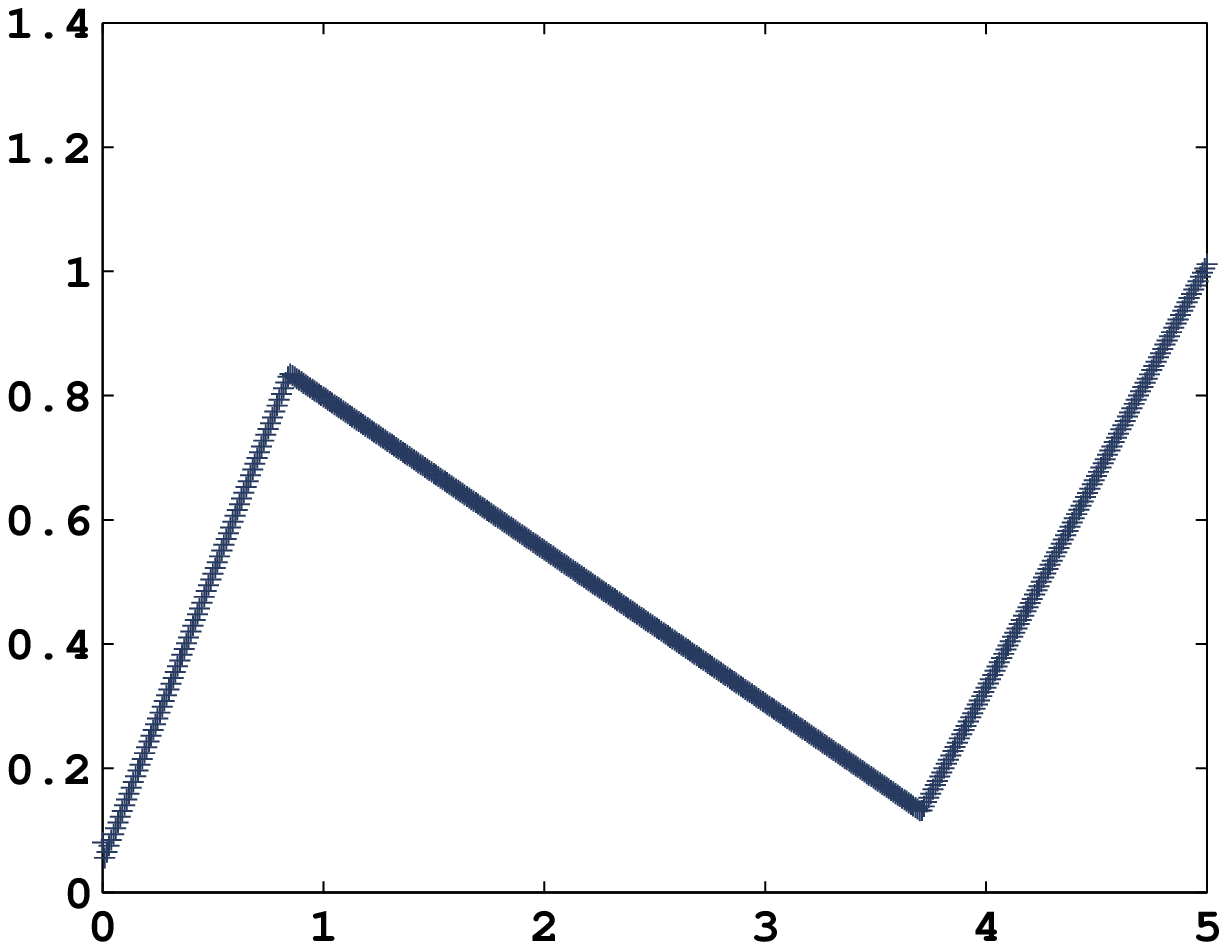}\\
(a) & (b) & (c)\\
\includegraphics[scale=.38]{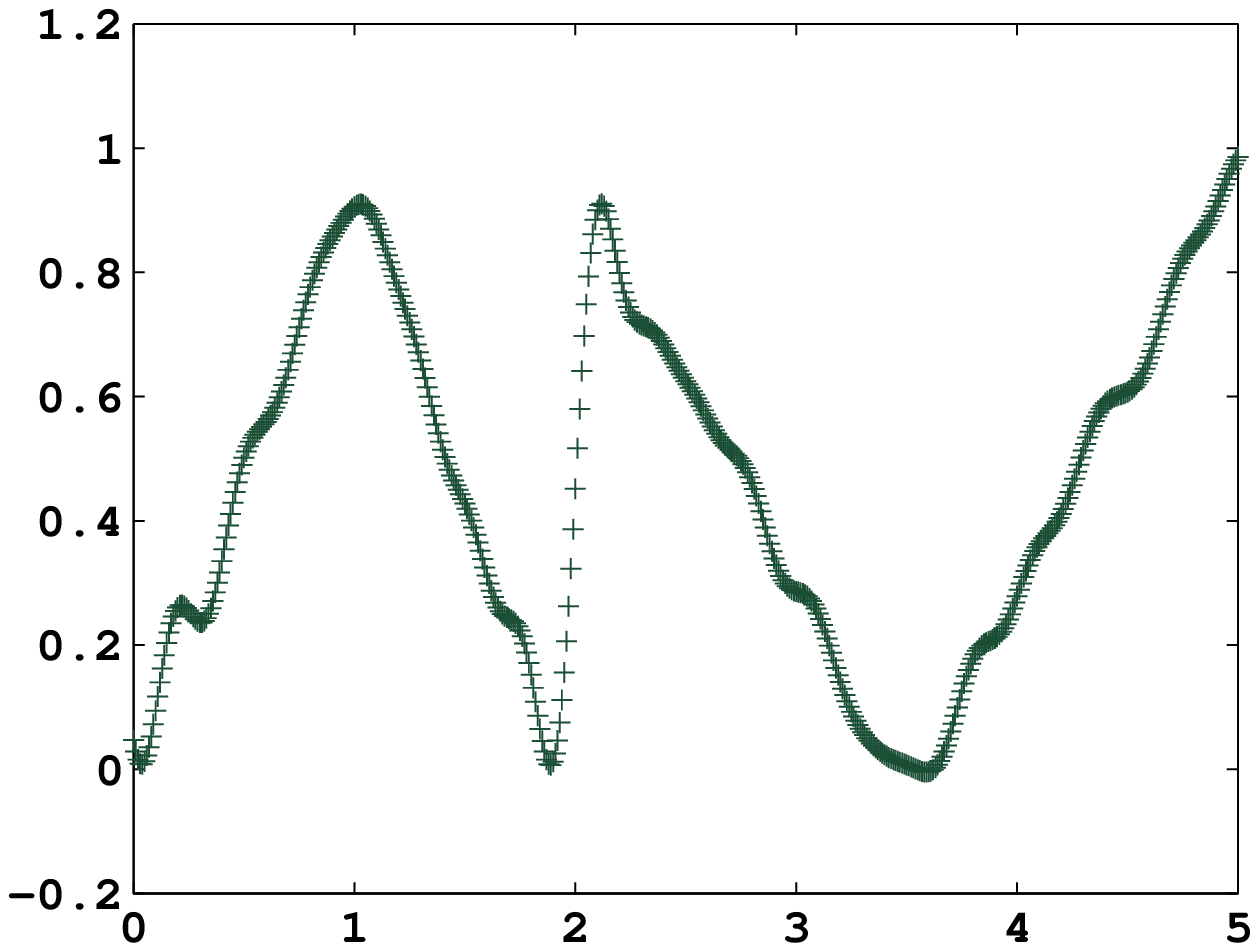} & \includegraphics[scale=0.36]{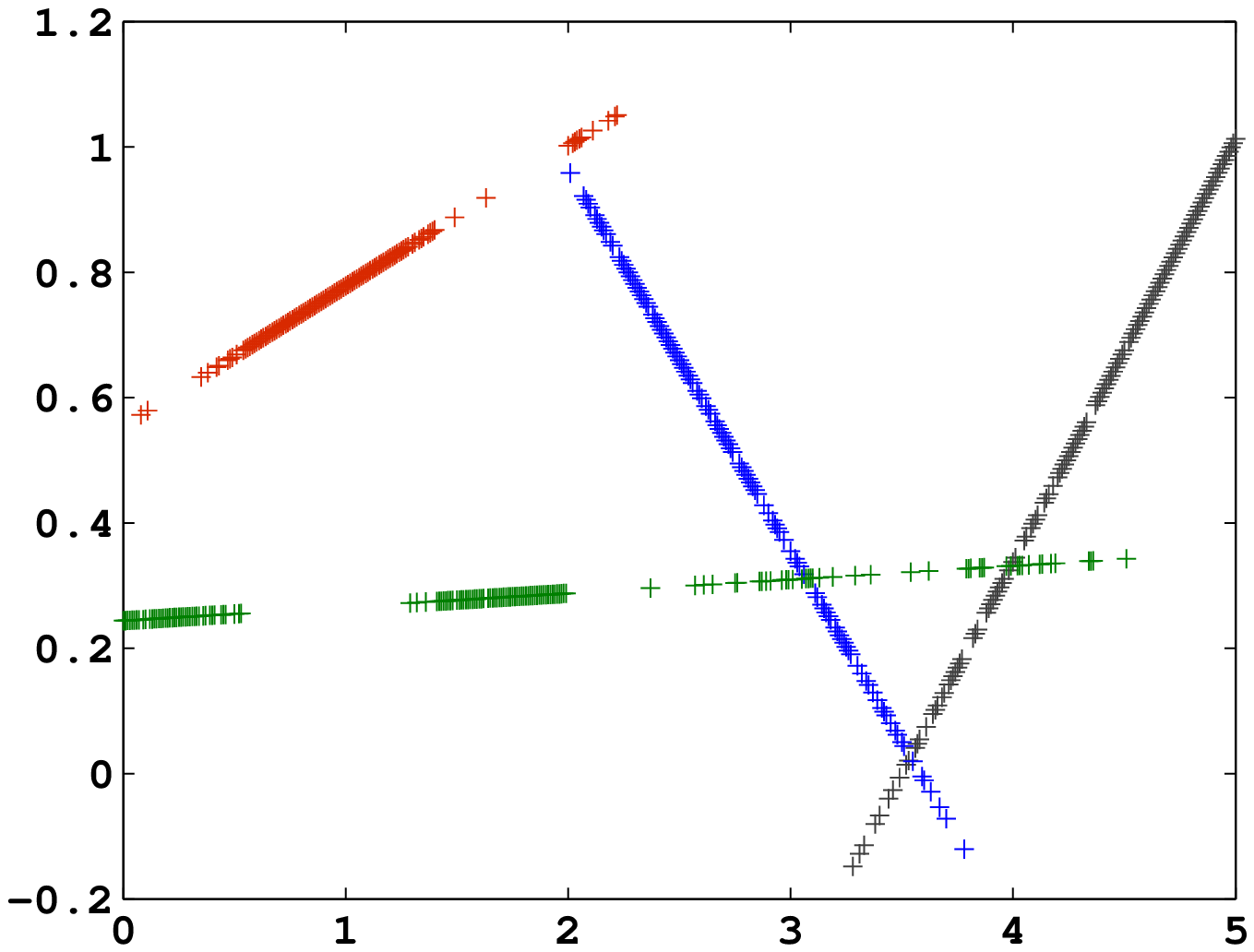} & \includegraphics[scale=.36]{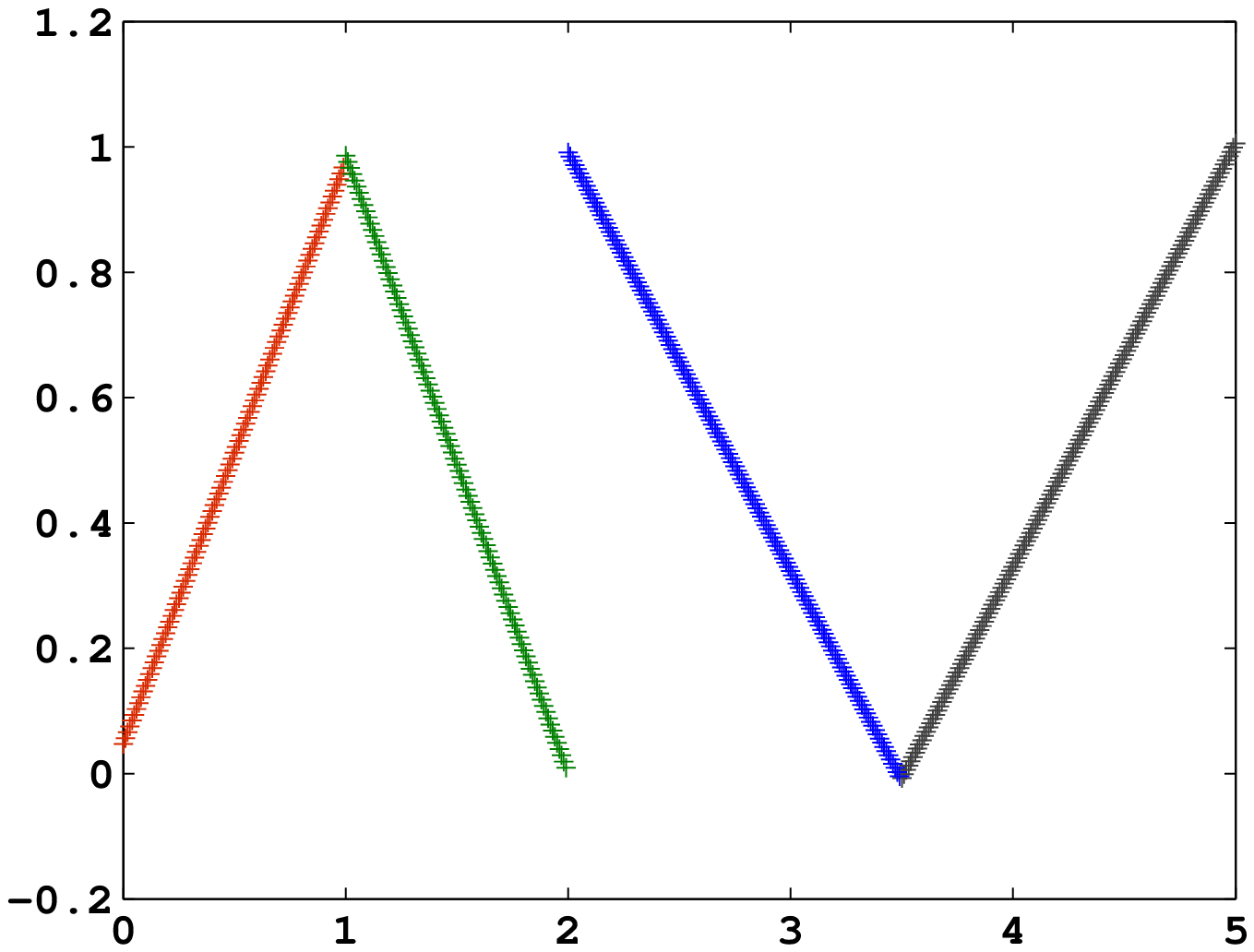}\\
(d) & (e) & (f)
  \end{tabular}
\caption{(a) 4-piecewise affine function $f$ described in problem 1; (b) function $f$ corrupted by additive
Gaussian noise; functions learnt using (c) hinge hyperplane algorithm, (d) support vector regression,
(e) $K$-plane regression approach and (f) modified $K$-plane regression approach, given noisy samples of $f$.}
\label{fig:1d-example1}
 \end{center}
\end{figure*}

\textbf{Problem 2: } In this problem the target function is a 3-piecewise affine function and we show the functions learnt by different approaches on
noise-free as well as noisy training set.
Figure \ref{fig:1d-example2} shows functions learnt using different approaches given the noise-free training examples.
As can be seen, all algorithms learn a very good approximation of the target function (when given noise-free training data).
The hinge hyperplanes method and SVR learn a continuous approximation while the K-plane and modified K-plane methods learn
the discontinuous function.

We see that both $K$-plane and modified $K$-plane regression approach learn
the function exactly. But MSE of $K$-plane method is much higher than modified $K$-plane method as can be seen in
Table \ref{Table:problem2}. The reason is as follows. In problem 2, the
three sets (defining the partitioning of the domain of the function)
corresponding to the three affine functions are [0,1), [1,2) and [2,3].
Moreover, $\forall x\in [0,1)\cup [2,3]$, $f(x)=x$. Thus the same affine function is assigned to two of the three disjoint
sets. The $K$-plane regression approach tries to partition the training data so that for each partition we can learn a good
function to fit the data; it does not care about whether the points in a partition are close together. Hence, any partition of the
 the set $[0,1) \cup [2,3]$ into two parts (including the case where one part is null) would result in roughly the same value
 of the error function for the $K$-plane method. But, for prediction on a new point, we have
to use the nearness of the new point to centroid of the
partitions. Hence, if the partitions are bad then the final MSE can be very large. In this problem, when $K$-plane regression is given noise free
samples, it always learnt only two hyperplanes irrespective of the value of $K$ used (with the sets ($S_k$) corresponding
to the remaining partitions being empty). The means of the two partitions learnt were $1.505$ and $1.4975$. This clearly shows the
algorithm has put  $[0,1) \cup [2,3]$ into one partition. This leads to very poor prediction on test samples and high MSE in case of
$K$-plane regression. In contrast, the means learnt using modified $K$-plane regression are $0.495$, $1.5$ and $2.505$.

In the second part of this problem, we have added noise to the true function values in the training
set as explained earlier. Figure \ref{fig:1d-example3}
shows the functions learnt using different approaches given these noisy samples of the function. The MSE achieved by the
learnt function on a test set under different algorithms are shown in  Table \ref{Table:problem2}.
We see that only modified $K$-plane regression approach learns the target function exactly.

The function learnt by $K$-plane regression is very poor
 and its MSE is also high. This shows that, unlike in the earlier case, the $K$-plane regression algorithm could not even get the
two affine functions correctly. Given the shape of function learnt on this problem by $K$-plane regression when the examples are
noise-free, we can see that this algorithm is very sensitive to additive noise.

Both hinge hyperplanes method and SVR learn a good continuous approximation to the target function. However, these are not as good
as the functions learnt by these algorithms on noise-free data of this problem.
In contrast, the modified $K$-plane regression algorithm learns the discontinuous function
exactly under our additive noise also. It also achieves the
minimum MSE which is nothing but the noise variance as can be seen in Table \ref{Table:problem2}.

\begin{table}
\begin{center}
  \begin{tabular}{|p{.4in}|p{.85in}|p{.9in}|p{.55in}|}
\hline
 & Method      &     Parameters        &      MSE          \\ \hline
Without   & $K$-plane   &     \# hyperplanes = 3 &    0.7917            \\
Noise & Modified $K$-plane &     \# hyperplanes = 3 &   3.33$\times 10^{-28}$       \\
& Hinge Hyperplane & \# hinges = 13 &   0.008\\
& SVR  &$C=1024$, $\sigma=16$, $\epsilon=2^{-8}$ &   0.0041 \\
\hline
With   & $K$-plane   &     \# hyperplanes = 3 &    0.1352            \\
Noise & Modified $K$-plane &     \# hyperplanes = 3 &   0.011       \\
& Hinge Hyperplane & \# hinges = 23 &   0.0237\\
& SVR  &$C=1024$, $\sigma=16$, $\epsilon=2^{-8}$ &   0.0148 \\
\hline
\end{tabular}
\caption{MSE of different regression approaches on problem 2.} \label{Table:problem2}
\end{center}
\end{table}

\begin{figure*}
 \begin{center}
  \begin{tabular}{ccc}
 \includegraphics[scale=.38]{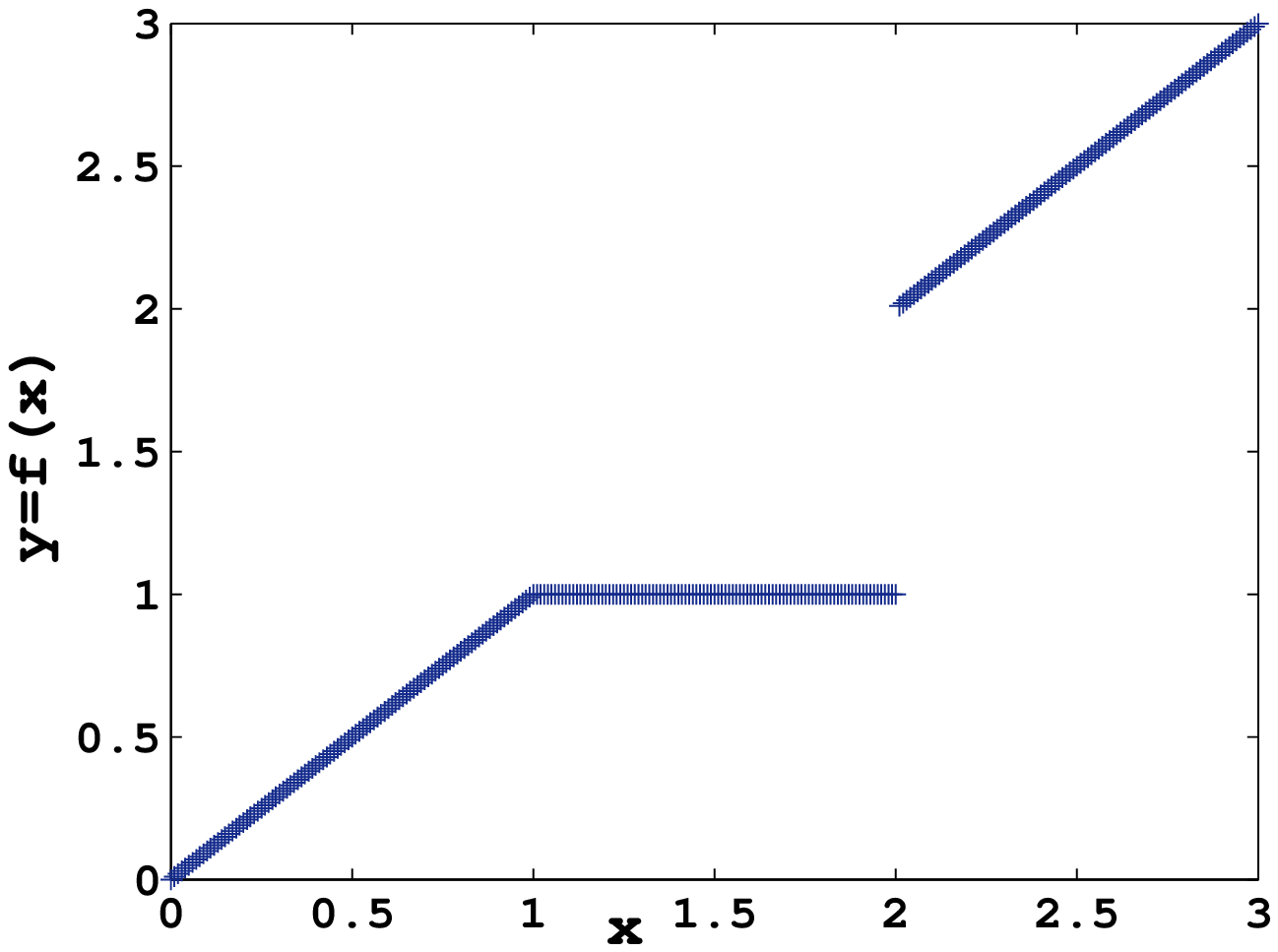}  &  \includegraphics[scale=.36]{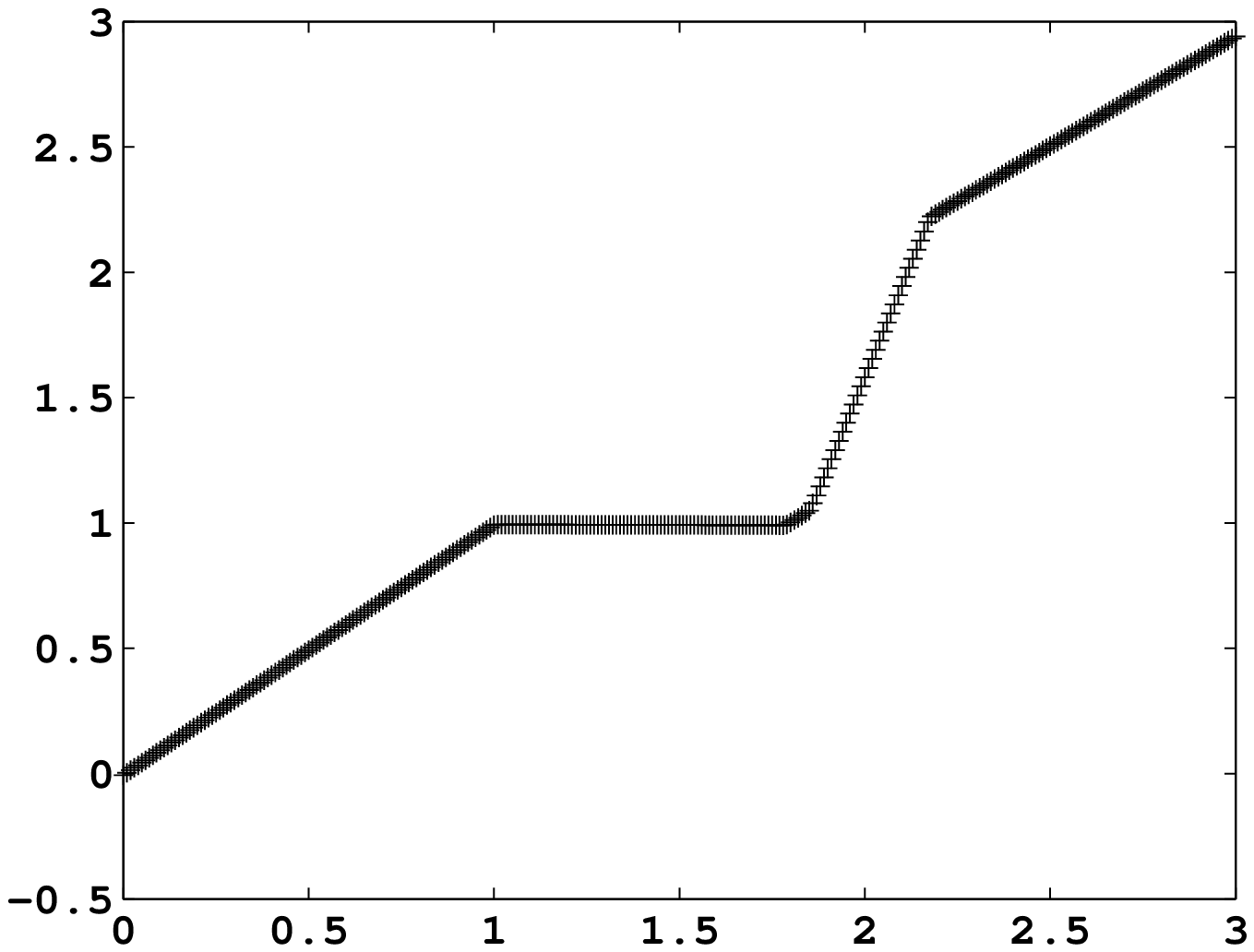}  & \includegraphics[scale=.38]{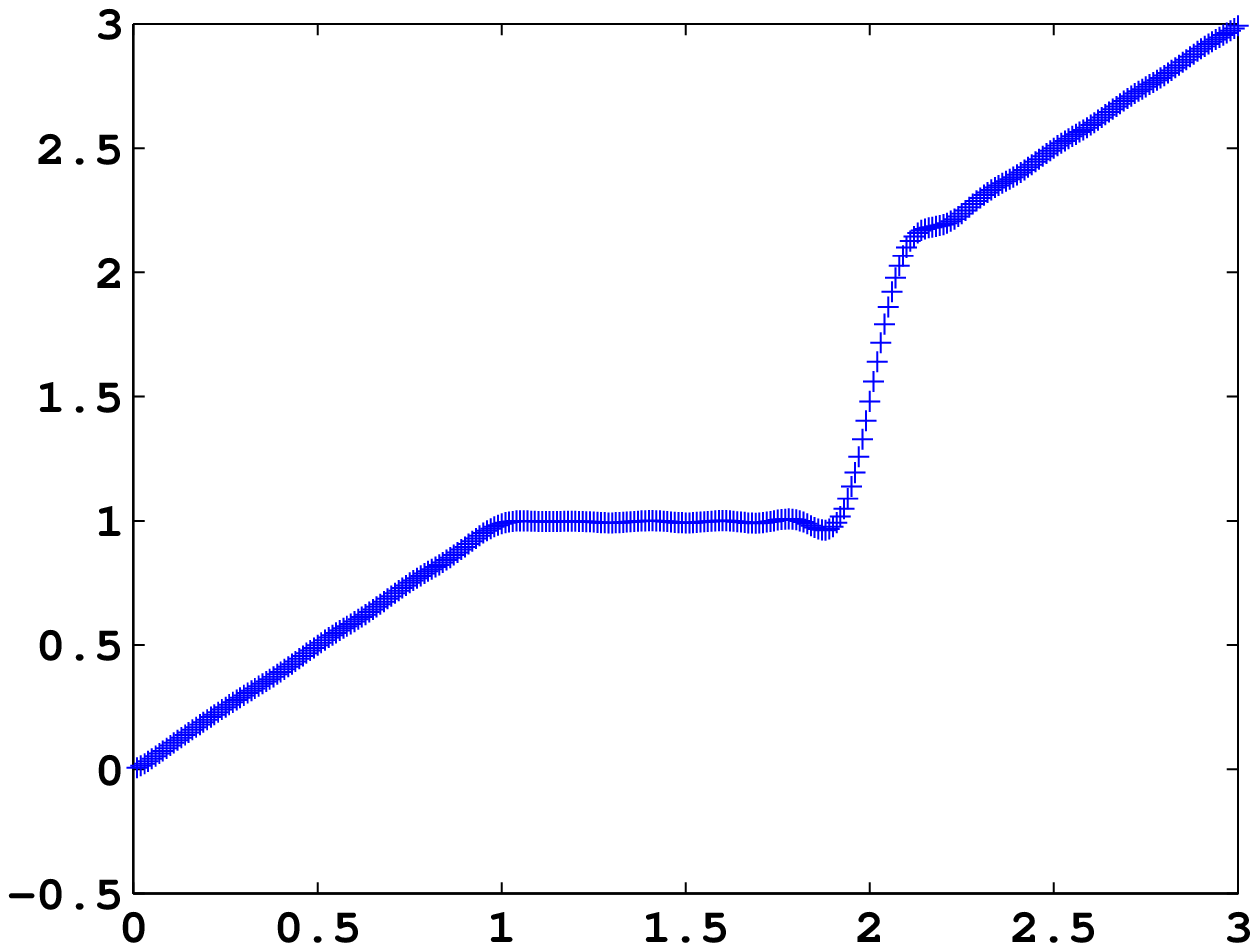} \\
(a)  & (b)  &(c) \\
\includegraphics[scale=.38]{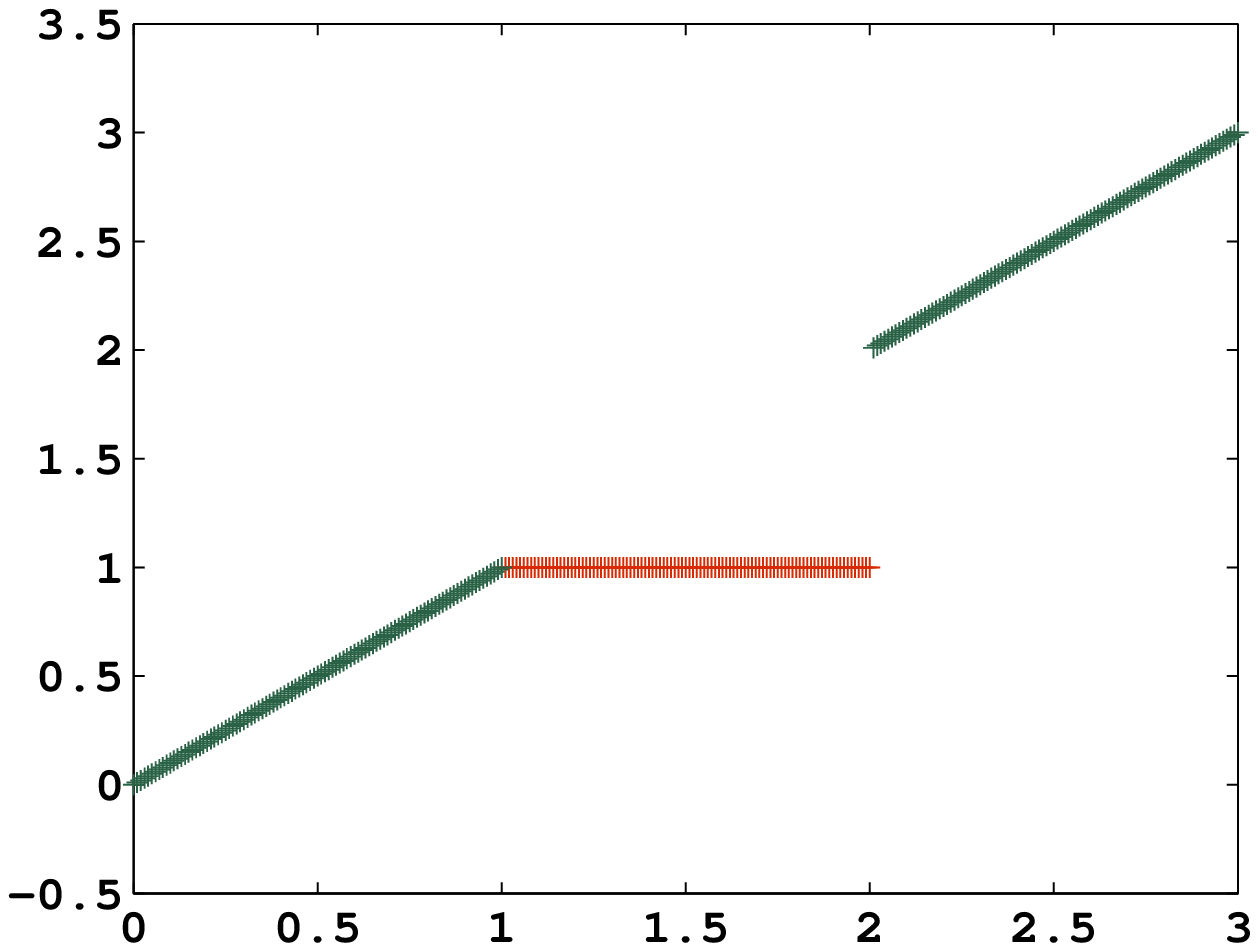} & \includegraphics[scale=0.38]{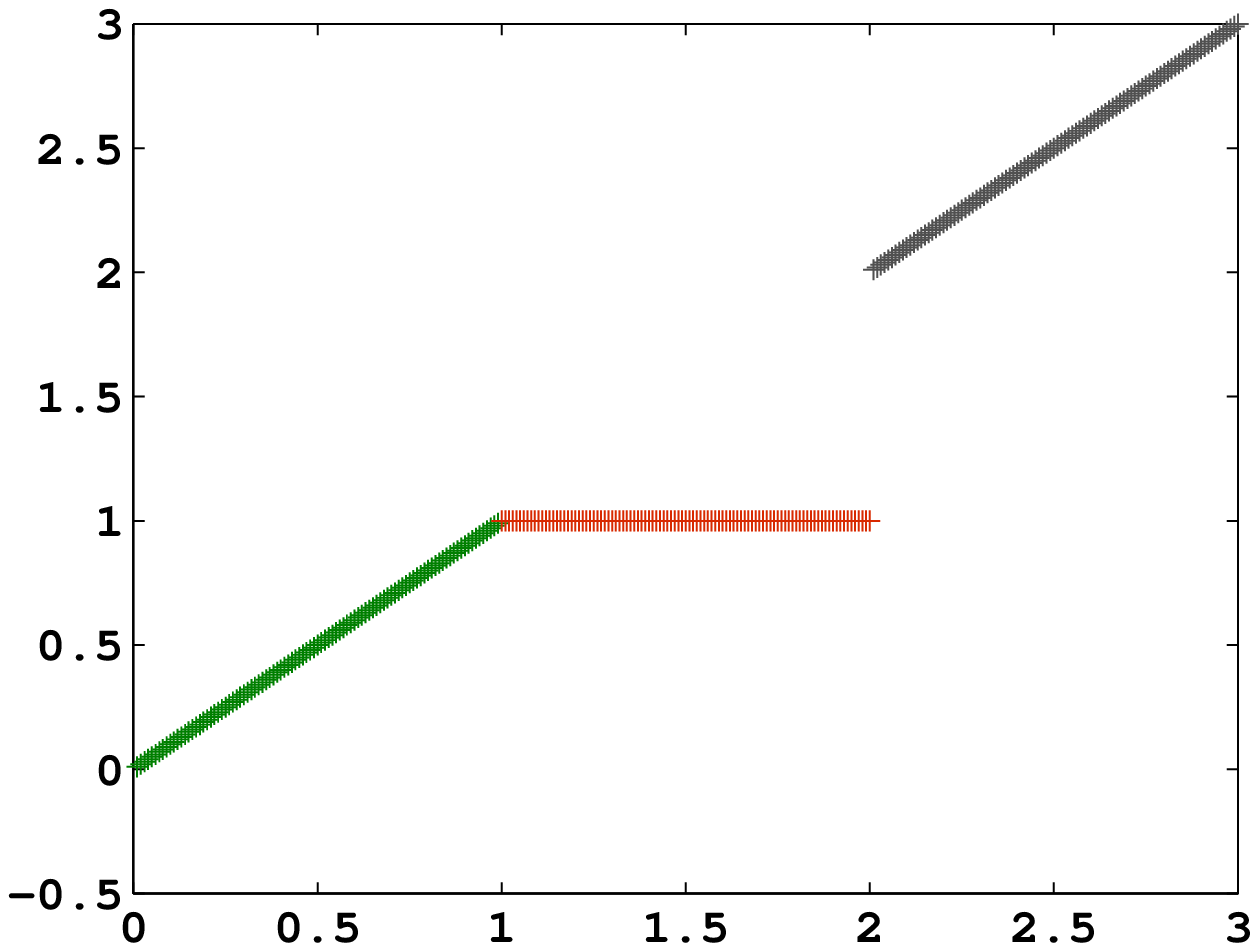} & \\
(d)  & (e) &
  \end{tabular}
\caption{(a) 3-piecewise affine function $f$ described in problem 2; functions learnt using (b) hinge hyperplane algorithm, (c) support vector regression,
(d) $K$-plane regression approach and (e) modified $K$-plane regression approach, given noise-free samples of $f$.}
\label{fig:1d-example2}
 \end{center}
\end{figure*}

\begin{figure*}
 \begin{center}
  \begin{tabular}{ccc}
 \includegraphics[scale=.4]{3plane-new.eps}  &  \includegraphics[scale=.38]{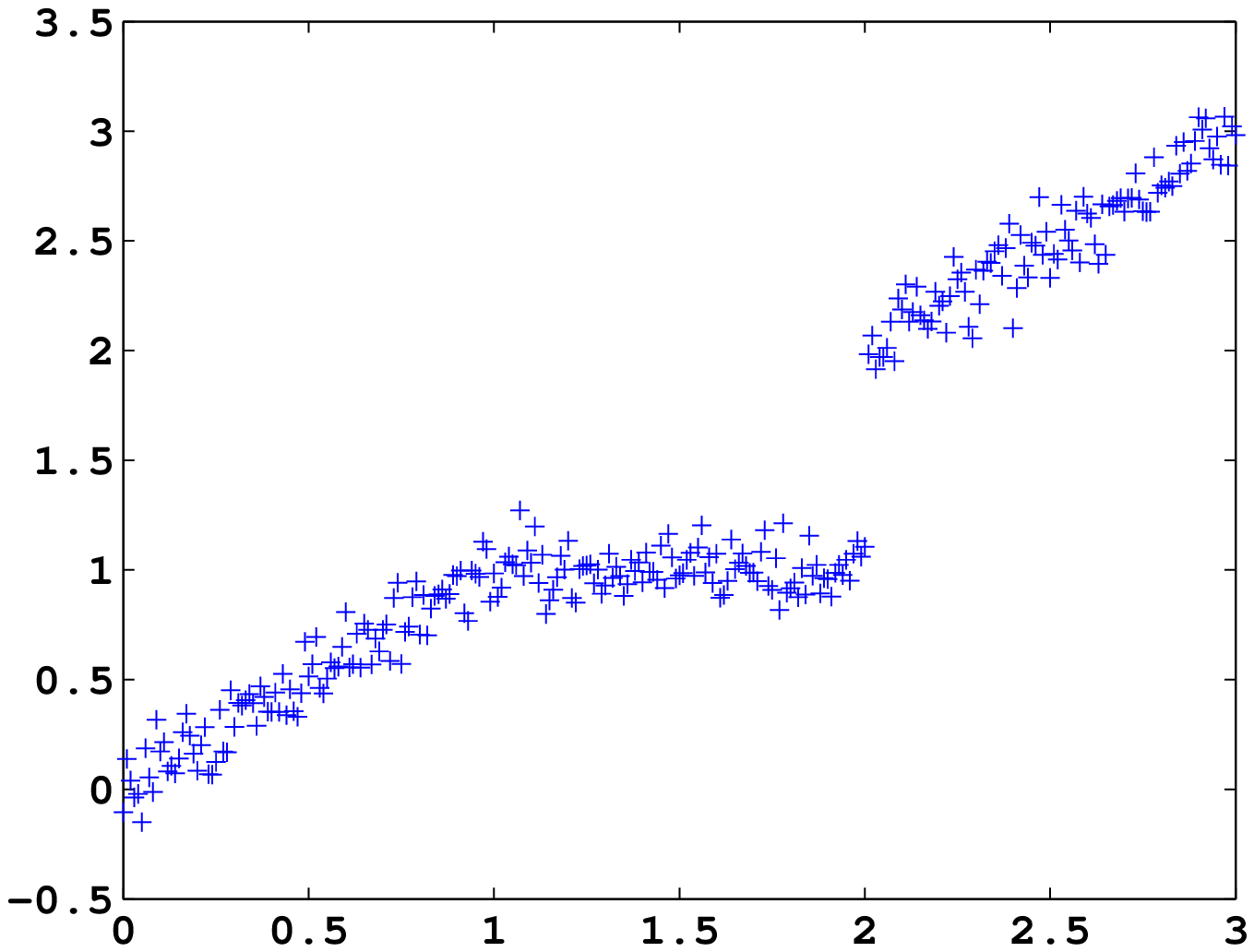}  & \includegraphics[scale=.38]{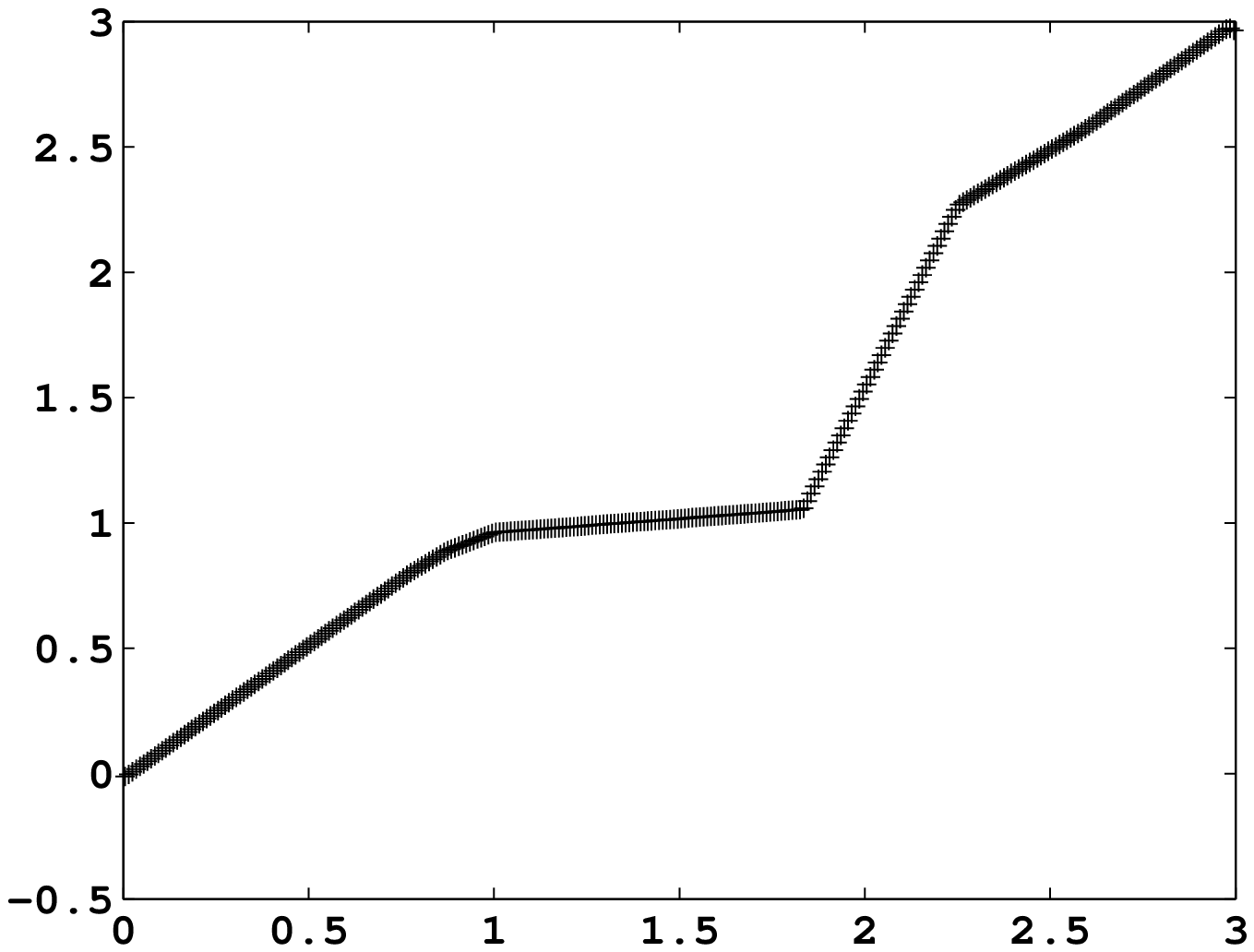}\\
(a)  &  (b)  &  (c) \\
 \includegraphics[scale=.38]{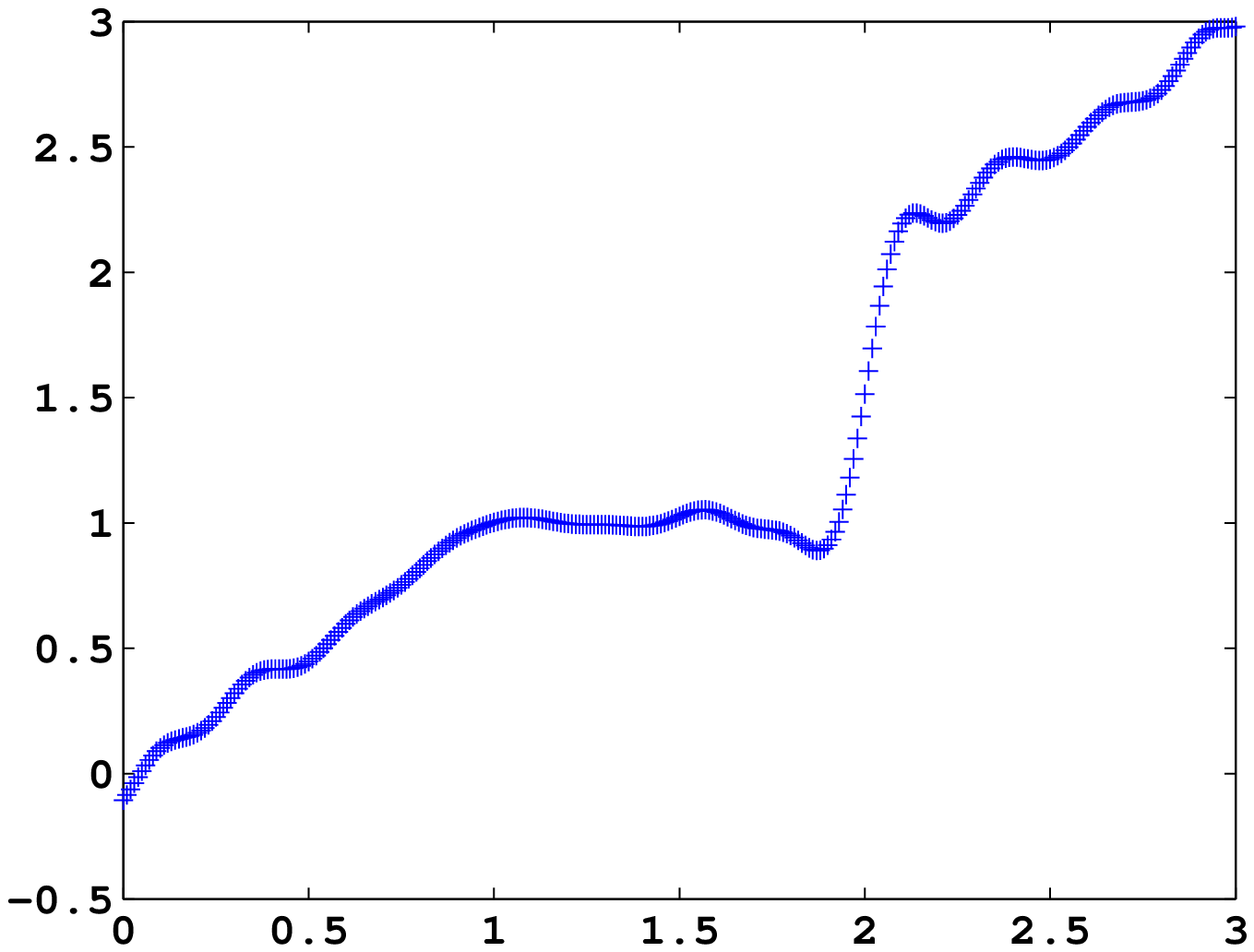}  & \includegraphics[scale=0.38]{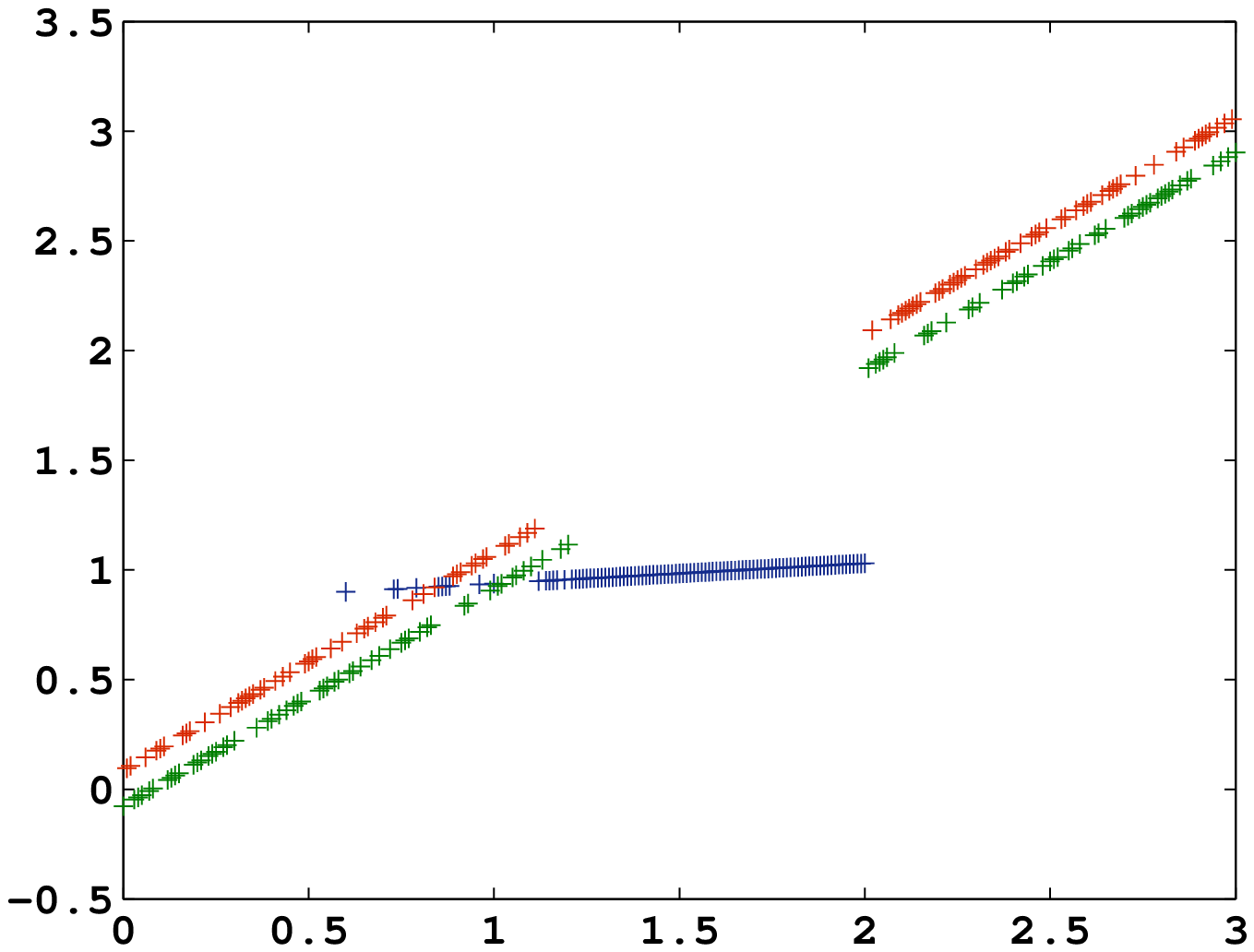} & \includegraphics[scale=0.38]{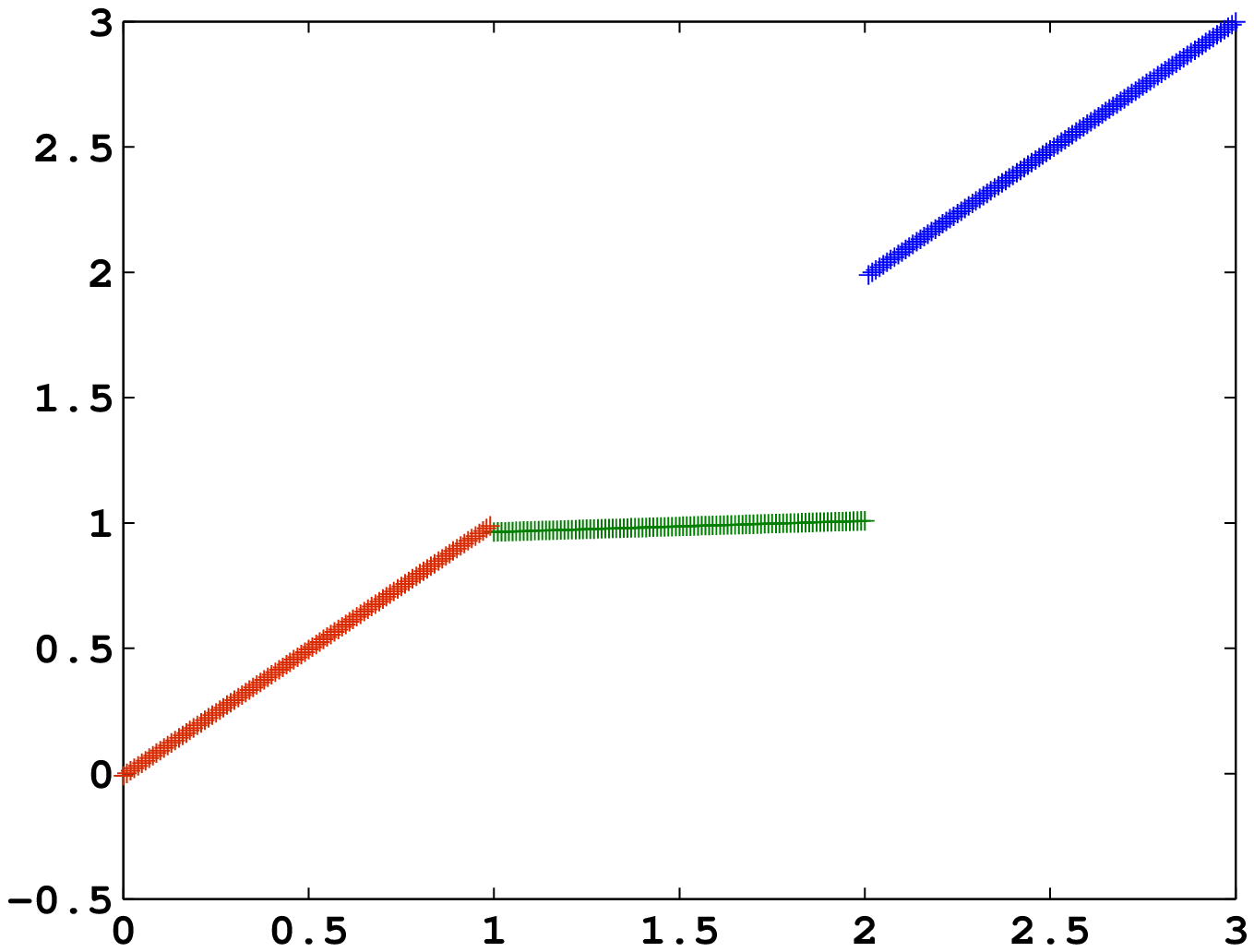}\\
9d)  &  (e) & (f)
  \end{tabular}
\caption{(a) 3-piecewise affine function $f$ described in problem 2; (b) function $f$ corrupted by additive Gaussian noise; functions learnt using (c) hinge hyperplane algorithm, (d) support vector regression,
(e) $K$-plane regression approach and (f) modified $K$-plane regression approach, given noisy samples of $f$.}
\label{fig:1d-example3}
 \end{center}
\end{figure*}

\subsection*{Results on Real Datasets: }
We now discuss performance of modified $K$-plane regression algorithm in comparison
with other approaches on different real datasets. The results provided are based on 10-
repetitions of 10-fold cross validation. We show average values and standard deviation of mean square error (MSE)
and the time taken. The results are presented in Table \ref{Table:comparison1}-\ref{Table:comparison5}.

\begin{table}
\begin{center}
  \begin{tabular}{|p{.85in}|p{.8in}|p{.6in}|p{0.55in}|}
\hline
Method      &     Parameters         &      MSE          &         Time(sec)           \\ \hline
$K$-plane   &     \# hyperplanes = 2 &   17.15$\pm$0.85  &  0.01                       \\
            &     \# hyperplanes = 3 &   27.47$\pm$2.74  &  0.03$\pm$0.003             \\
            &     \# hyperplanes = 4 &   30.29$\pm$1.77  &  0.04$\pm$0.005             \\
            &     \# hyperplanes = 5 &   39.67$\pm$9.11  &  0.06$\pm$0.012             \\
\hline
Modified $K$-plane  &     \# hyperplanes = 2 &  14.95$\pm$0.27  &  0.006               \\
$\gamma=100$        &     \# hyperplanes = 3 &  14.72$\pm$0.53  &  0.01$\pm$0.001      \\
                    &     \# hyperplanes = 4 &  14.25$\pm$0.62  &  0.014$\pm$0.001     \\
                    &     \# hyperplanes = 5 &  13.92$\pm$0.78  &  0.02$\pm$0.002      \\
\hline
Hinge Hyperplane & \# hinges = 1 & 19.29$\pm$2.19  &  0.01$\pm$0.003 \\
& \# hinges = 2 & 16.45$\pm$1.34 & 0.04$\pm$0.006  \\
& \# hinges = 3 & 16.25$\pm$1.16 & 0.07$\pm$0.006  \\
& \# hinges = 4 & 16.06$\pm$0.98 & 0.11$\pm$0.008  \\
& \# hinges = 5 & 15.62$\pm$1.57 & 0.14$\pm$0.015  \\
\hline
SVR  &$C=128$, $\sigma=0.25$, $\epsilon=2^{-8}$ &   10.08$\pm$0.42 &  0.17$\pm$0.01 \\
\hline
\end{tabular}
\caption{Comparison results of modified $K$-plane regression approach with other regression approaches on Housing Dataset.} \label{Table:comparison1}
\end{center}
\end{table}

\begin{table}
\begin{center}
  \begin{tabular}{|p{.85in}|p{.8in}|p{.6in}|p{0.55in}|}
\hline
Method             &     Parameters        &      MSE          &         Time(sec)    \\ \hline
$K$-plane          &      \# hyperplanes=2 &  10.44$\pm$0.17    &  0.11$\pm$0.02        \\
                   &      \# hyperplanes=3 &  9.19$\pm$0.62     &  0.14$\pm$0.01        \\
                   &      \# hyperplanes=4 &  9.87$\pm$1.45     &  0.27$\pm$0.03        \\
                   &      \# hyperplanes=5 &  10.85$\pm$1.19    &  0.39$\pm$0.05        \\
\hline
Modified $K$-plane &      \# hyperplanes=2&  4.80$\pm$0.02  &  0.02$\pm$0.002        \\
$\gamma$=100 &      \# hyperplanes=3&  4.68$\pm$0.03  &  0.04$\pm$0.005        \\
 &      \# hyperplanes=4&  4.69$\pm$0.03  &  0.08$\pm$0.01        \\
 &      \# hyperplanes=5&  4.68$\pm$0.03  &  0.08$\pm$0.01        \\
\hline
Hinge Hyperplane & \# hinges = 1 & 4.73$\pm$0.06  &  0.01$\pm$0.001         \\
& \# hinges = 2  &  4.53$\pm$0.03& 0.08$\pm$0.01             \\
& \# hinges = 3  &  4.47$\pm$0.04& 0.17$\pm$0.02             \\
& \# hinges = 4  &  4.41$\pm$0.02& 0.28$\pm$0.02             \\
& \# hinges = 5  &  4.44$\pm$0.06& 0.40$\pm$0.03             \\
\hline
SVR   & $C=32$, $\sigma=0.5$, $\epsilon=0.5$ &  4.50$\pm$0.01 & 1.68$\pm$0.01          \\
\hline
\end{tabular}
\caption{Comparison results of modified $K$-plane regression approach with other regression approaches on Abalone Dataset.} \label{Table:comparison2}
\end{center}
\end{table}

\begin{table}
\begin{center}
  \begin{tabular}{|p{.85in}|p{.8in}|p{.6in}|p{0.55in}|}
\hline
Method             &     Parameters        &      MSE          &         Time(sec)    \\ \hline
$K$-plane           &      \# hyperplanes=2 &  10.34$\pm$0.22  &  0.02$\pm$0.001        \\
           &      \# hyperplanes=3 &  11.15$\pm$0.56           &  0.02$\pm$0.003        \\
           &      \# hyperplanes=4 &  13.08$\pm$1.10  &  0.04$\pm$0.003        \\
           &      \# hyperplanes=5 &  13.72$\pm$0.77  &  0.05$\pm$0.003        \\
\hline
Modified $K$-plane  &      \# hyperplanes=2 &  8.55$\pm$0.11  &  0.006        \\
$\gamma$=100  &      \# hyperplanes=3 &  8.72$\pm$0.25  &  0.01$\pm$0.003        \\
  &      \# hyperplanes=4 &  8.82$\pm$0.75  &  0.01$\pm$0.001        \\
  &      \# hyperplanes=5 &  8.83$\pm$0.69  &  0.01$\pm$0.002        \\
\hline
Hinge Hyperplane & \# hinges = 1 & 9.81$\pm$0.52  &  0.003    \\
&  \# hinges = 2 & 9.03$\pm$0.53 & 0.02$\pm$0.002             \\
&  \# hinges = 3 & 8.75$\pm$0.37 & 0.03$\pm$0.01              \\
&  \# hinges = 4 & 8.58$\pm$0.35 & 0.05$\pm$0.009             \\
&  \# hinges = 5 & 8.35$\pm$0.39 & 0.08$\pm$0.005             \\
\hline
SVR   & $C$=16, $\sigma=1$, $\epsilon=0.25$ & 6.80$\pm$0.26 & 0.03         \\
\hline
\end{tabular}
\caption{Comparison results of modified $K$-plane regression approach with other regression approaches on Auto-mpg Dataset.} \label{Table:comparison4}
\end{center}
\end{table}

\begin{table}
\begin{center}
  \begin{tabular}{|p{.85in}|p{.8in}|p{.6in}|p{0.55in}|}
\hline
Method             &     Parameters        &      MSE          &         Time(sec)    \\ \hline
$K$-plane           &      \# hyperplanes=2 &  61.81$\pm$7.17  &  0.39$\pm$0.05        \\
           &      \# hyperplanes=3 &  19.48$\pm$0.49  &  0.48$\pm$0.07        \\
           &      \# hyperplanes=4 &  15.88$\pm$0.92  &  0.92$\pm$0.13        \\
           &      \# hyperplanes=5 &  19.98$\pm$1.03  &  1.19$\pm$0.20        \\
\hline
Modified $K$-plane  &      \# hyperplanes=2 & 154.45$\pm$12.61  &  0.15$\pm$0.02        \\
$\gamma=100$  &      \# hyperplanes=3 &  23.47$\pm$1.56  &  0.24$\pm$0.03       \\
  &      \# hyperplanes=4 &  11.88$\pm$0.15  &  0.48$\pm$0.04        \\
  &      \# hyperplanes=5 &  11.80$\pm$0.24  &  0.66$\pm$0.10        \\
  &      \# hyperplanes=6 &  10.98$\pm$0.14  &  0.70$\pm$0.16  \\
\hline
Hinge Hyperplane & \# hinges = 1 & 29.40$\pm$21.57  &  0.05$\pm$0.002  \\
& \# hinges = 2 & 11.39$\pm$0.32 & 0.17$\pm$0.017 \\
& \# hinges = 3 & 10.77$\pm$0.27 & 0.38$\pm$0.047 \\
& \# hinges = 4 & 10.66$\pm$0.33 & 0.64$\pm$0.062 \\
& \# hinges = 5 & 10.06$\pm$0.13 & 0.98$\pm$0.055 \\
\hline
SVR & $C=256$, $\sigma=1$, $\epsilon=0.5$  & 8.47$\pm$0.11 & 23.16$\pm$0.31 \\
\hline
\end{tabular}
\caption{Comparison results of modified $K$-plane regression approach with other regression approaches on Computer Activity Dataset.} \label{Table:comparison5}
\end{center}
\end{table}

We see that for all datasets, the MSE achieved by the simple $K$-plane regression method is highest among all algorithms. Thus, though the
$K$-plane regression method is conceptually simple and appealing, its performance is not very good.

The modified $K$-plane regression algorithm performs much better than $K$-plane regression not only in terms of MSE but also in terms of time
taken. The reason why modified $K$-planes method takes lesser time is that it converges in fewer iterations.
This happens because modified $K$-plane regression algorithm gives importance to the connectedness of the clusters also.
 As a result, number of transitions of points from one cluster to another after every iteration
are lesser and thus the clusters stabilize after fewer iterations.


The performance of modified $K$-plane regression algorithm is comparable to that of hinge hyperplane algorithm in terms of MSE.
It performs better than hinge hyperplane method on Housing dataset. On Auto-Mpg dataset, Abalone dataset and Computer Activity dataset,
the minimum MSE of modified $K$-plane regression approach is only a little higher than the minimum MSE of hinge hyperplane method.
Modified $K$-plane regression algorithm is also faster
than hinge hyperplane method on all data sets.

On all problems except on the Abalone dataset, the SVR algorithms achieves better MSE than modified $K$-plane regression
algorithm. However, we observe that modified $K$-plane regression is significantly faster than SVR.
In SVR, the
complexity of dual optimization problem is $O(N^3)$, where $N$ is the number of points.
In contrast, in modified $K$-plane regression, at each iteration, the major computation is finding
$K$ linear regression functions. The time complexity of each iteration in modified $K$-plane regression
is $O(K(d+1)^3)$ which is very less than $O(N^3)$ if $N>>d$.



Thus, we see that overall, modified $K$-plane regression is a very attractive method for learning nonlinear regression functions
by approximating them as piecewise linear functions. It is conceptually simple and the algorithm is very efficient. Its performance is
comparable to that of SVR or hinge hyperplanes method in terms of accuracy. It is significantly faster than SVR and is also faster than hinge
hyperplane method. Further, unlike all other current regression function learning algorithms, this method is capable of learning
discontinuous functions also.

\section{Conclusions}
\label{sec:conclusion}

In this paper, we considered the problem of learning piecewise linear regression models.
We proposed an interesting and simple algorithm to learn such functions. The proposed method
is capable of learning discontinuous functions also. Through simulation experiments we showed that the
performance of the proposed method is good and is comparable to state-of-art in regression function learning.

The basic idea behind the proposed method is very simple.
Let $S=\{(\xx_1 , y_1),\ldots,(\xx_N,y_N)\}$ be the training dataset.
We essentially want to find a way to partition the set $\{ \xx_1, \cdots, \xx_N \}$ into $K$
sets such that we can find a good linear fit for the targets (i.e., $y_i$) of points in each partition.
 The algorithm achieves this by repeatedly partitioning the points and fitting
linear models. After each model fit, we repartition the points based on the closeness of targets to the current models.
We called this the $K$-plane regression algorithm.
This algorithm is conceptually simple and is similar in spirit to the $K$-means clustering method.
While such an idea has been discussed in different contexts, we have not come across this algorithm
proposed and empirically investigated for nonlinear regression.

Though this idea is interesting, as we showed here, it has several drawbacks. As the results in previous section show,
this algorithm performs poorly even on one dimensional problems.

In this paper we have also proposed a modification of the above method which performs well as a regression learning method.
In our modified $K$-plane regression algorithm, during the process of repeatedly partitioning feature vectors and fitting linear models,
we make the partitions so that we get good linear models and, also, the points of a partition are all close together. This
idea is easily incorporated into the algorithm by expanding the parameter vector to be learnt and by modifying the objective
function to be minimized. The resulting algorithm essentially does one step of linear regression and one step of $K$-means
clustering in each iteration.

Through empirical studies, we showed that the modified $K$-plane regression algorithm is very effective. Its performance
on some real data sets is comparable to that of nonlinear SVR in terms of accuracy while the proposed method is much faster than SVR.
The proposed method is better than the hinge hyperplane algorithm which is arguably the best method today for learning piecewise linear
functions. Through two synthetic one-dimensional problems, we also showed that the proposed method has better robustness to additive noise
than the other methods and that it is capable of learning discontinuous functions also.

We feel that the proposed method opens up interesting possibilities of designing algorithms for learning piecewise
linear functions. As mentioned earlier, simultaneous estimation of optimal partitions and optimal models for each partition is
computationally intractable. Hence an interesting and difficult open question is to establish theoretical bounds on the performance of
the modified $K$-plane regression method. While we showed that the method can be viewed as a limiting case EM algorithm under
reasonable probability model, a lot of work needs to be done to understand how close to optimum  can such methods converge to.

\bibliography{proposal}

\begin{thebibliography}{10}

\bibitem{Bishop}
C.~M. Bishop, {\em Pattern Recognition and Machine Learning (Information
  Science and Statistics)}.
\newblock Springer, first~ed., 2006.

\bibitem{Hastie}
T.~Hastie, R.~Tibshirani, and J.~Friedman, {\em Elements of Statistical
  Learning Theory}.
\newblock Springer, 2001.

\bibitem{smola1998}
A.~J. Smola and B.~Sch\"{o}lkopf, ``{A Tutorial on Support Vector
  Regression},'' NeuroCOLT2 Technical Report Series NC2-TR-1998-030, GMD,
  October 1998.

\bibitem{Zhu2008}
J.~Zhu, S.~C.~H. Hoi, and M.~R.-T. Lyu, ``Robust regularized kernel
  regression,'' {\em IEEE Transactions on Systems, Man and Cybernetics-Part B:
  Cybernetics}, vol.~38, pp.~1639--1644, December 2008.

\bibitem{Jin2005}
J.-T. Jeng, ``Hybrid approach of selecting hyperparameters of support vector
  machine for regression,'' {\em IEEE Transactions on Systems, Man and
  Cybernetics-Part B: Cybernetics}, vol.~36, pp.~699--709, June 2005.

\bibitem{Paoletti2007}
S.~Paoletti, A.~L. Juloski, Ferrari-Trecate, and R.~Vidal, ``{Identification of
  Hybrid Systems: A Tutorial},'' {\em European Journal of Control}, vol.~13,
  no.~2/3, pp.~242--260, 2007.

\bibitem{Breiman}
L.~Breiman, ``{Hinging Hyperplane for Regression, Classification and Function
  Approximation},'' {\em IEEE Transaction on Information Theory}, vol.~39,
  pp.~999--1013, May 1993.

\bibitem{Shuning}
S.~Wang and X.~Sun, ``{Generalization of Hinging Hyperplanes},'' {\em IEEE
  Transaction on Information Theory}, vol.~51, pp.~4425--4431, December 2005.

\bibitem{Magnani2009}
A.~Magnani and S.~P. Boyd, ``Convex piecewise linear fitting,'' {\em
  Optimization and Engineering}, vol.~10, pp.~1--17, March 2009.

\bibitem{Nowlan1}
R.~Jacobs, M.~Jordan, S.~Nowlan, and G.~Hinton, ``Adaptive mixtures of local
  experts,'' {\em Neural Computation}, vol.~3, pp.~79--87, March 1991.

\bibitem{Waterhouse}
S.~R. Waterhouse, {\em Classification and Regression using Mixtures of
  Experts}.
\newblock PhD thesis, Department of Engineering, University of Cambridge, 1997.

\bibitem{Jordan94}
M.~I. Jordan and R.~A. Jacobs, ``{Hierarchical Mixture of Experts and EM
  Algorithm},'' {\em Neural Computation}, vol.~6, pp.~181--214, March 1994.

\bibitem{Breiman1984}
L.~Breiman, J.~Friedman, R.~Olshen, and C.~Stone, {\em Classification and
  Regression Trees}.
\newblock Statistics/Probability Series, Belmont, California, U.S.A.: Wadsworth
  and Brooks, 1984.

\bibitem{Jayadeva}
A.~K.~D. Jayadeva and S.~Chandra, ``{Algorithm for Building a Neural Network
  for Function Approximation},'' {\em IEE Proc.-Circuits Devices Systems},
  vol.~149, pp.~301--307, October-December 2002.

\bibitem{Pucar}
P.~Pucar and J.~Sj\"{o}berg, ``{On the Hinge-Finding Algorithm for Hinging
  Hyperplanes},'' {\em IEEE Transaction on Information Theory}, vol.~44,
  pp.~1310--1319, May 1998.

\bibitem{Amaldi}
E.~Amaldi and M.~Mattavelli, ``{The MIN PFS Problem and Piecewise Linear Model
  Estimation},'' {\em Discrete Applied Mathematics}, vol.~118, pp.~115--143,
  April 2002.

\bibitem{Alberto2003}
A.~Bemporad, A.~Garulli, S.~Paoletti, and A.~Vicino, ``{A Greedy Approach to
  Identification of Piecewise Affine Models},'' in {\em Proceedings of the 6th
  International Conference on Hybrid Systems: Computation and Control (HSCC)},
  (Prague, Czech Republic), pp.~97--112, April 2003.

\bibitem{Alberto2005}
A.~Bemporad, A.~Garulli, S.~Paoletti, and A.~Vicino, ``{A Bounded Error
  Approach to Piecewise Affine System Identification},'' {\em IEEE Transaction
  on Automatic Control}, vol.~50, pp.~1567--1580, October 2005.

\bibitem{Xu1995}
L.~Xu, M.~I. Jordan, and G.~E. Hinton, ``An alternative model for mixtures of
  experts,'' in {\em Proceedings of Advances in Neural Information Processing
  Systems (NIPS)}, (Denver, CO, USA), pp.~633--640, November 1995.

\bibitem{Asuncion+Newman:2007}
D.~N. A.~Asuncion, {\em {UCI} Machine Learning Repository}.
\newblock University of California, Irvine, School of Information and Computer
  Sciences, 2007.
\newblock http://www.ics.uci.edu/$\sim$mlearn/{MLR}epository.html.

\bibitem{Chih2001}
C.-C. Chang and C.-J. Lin, {\em LIBSVM : A Library for Support Vector
  Machines}, 2001.
\newblock Software available at http://www.csie.ntu.edu.tw/~cjlin/libsvm.

\end{thebibliography}
\bibliographystyle{ieeetr}
\end{document}